\documentclass[11pt]{article}

\usepackage{amsmath, amsfonts, amssymb, amsthm,  amsbsy, graphicx, dsfont, mathtools, enumerate, enumitem}
\usepackage{multirow, float, algorithm, algpseudocode, changepage}
\usepackage{footnote}

\usepackage[blocks, affil-it]{authblk}

\usepackage[numbers, square]{natbib}
\usepackage[CJKbookmarks=true,
            bookmarksnumbered=true,
			bookmarksopen=true,
			colorlinks=true,
			citecolor=red,
			linkcolor=blue,
			anchorcolor=red,
			urlcolor=blue]{hyperref}
\usepackage[usenames]{color}

\usepackage[letterpaper, left=1.2truein, right=1.2truein, top = 1.2truein, bottom = 1.2truein]{geometry}

\usepackage{prettyref,soul}
\usepackage[font=small,labelfont=it]{caption}

\newtheorem{lemma}{Lemma}[section]
\newtheorem{proposition}{Proposition}[section]
\newtheorem{thm}{Theorem}[section]
\newtheorem{definition}{Definition}[section]

\newtheorem{condition}{Condition}[section]

\newrefformat{eq}{(\ref{#1})}
\newrefformat{chap}{Chapter~\ref{#1}}
\newrefformat{sec}{Section~\ref{#1}}
\newrefformat{algo}{Algorithm~\ref{#1}}
\newrefformat{fig}{Fig.~\ref{#1}}
\newrefformat{tab}{Table~\ref{#1}}
\newrefformat{rmk}{Remark~\ref{#1}}
\newrefformat{clm}{Claim~\ref{#1}}
\newrefformat{def}{Definition~\ref{#1}}
\newrefformat{cor}{Corollary~\ref{#1}}
\newrefformat{lmm}{Lemma~\ref{#1}}
\newrefformat{lemma}{Lemma~\ref{#1}}
\newrefformat{prop}{Proposition~\ref{#1}}
\newrefformat{app}{Appendix~\ref{#1}}
\newrefformat{ex}{Example~\ref{#1}}
\newrefformat{exer}{Exercise~\ref{#1}}
\newrefformat{soln}{Solution~\ref{#1}}
\newrefformat{cond}{Condition~\ref{#1}}

\def\text#1{\mbox{\rm #1}}

\newcommand{\argmin}{\mathop{\rm argmin}}
\newcommand{\argmax}{\mathop{\rm argmax}}

\newcommand{\norm}[1]{\|{#1} \|}

\newcommand{\wh}{\widehat}
\newcommand{\wt}{\widetilde}

\newcommand{\opnorm}[1]{\|#1\|_{\rm op}}

\newcommand{\Expect}{\mathbb{E}}

\newcommand{\TV}{{\sf TV}}
\newcommand{\JS}{{\sf JS}}
\newcommand{\sig}{{\sf sigmoid}}
\newcommand{\relu}{{\sf ReLU}}
\newcommand{\ramp}{{\sf ramp}}

\newcommand{\floor}[1]{{\left\lfloor {#1}\right\rfloor}}

\title{Generative Adversarial Nets for Robust Scatter Estimation:\\
A Proper Scoring Rule Perspective
}
\author{Chao Gao$^1$, Yuan Yao$^2$ and Weizhi Zhu$^2$\\
~\\
$^1$University of Chicago and $^2$Hong Kong University of Science and Technology
}

\begin{document}
\maketitle

\begin{abstract}

Robust scatter estimation is a fundamental task in statistics. The recent discovery on the connection between robust estimation and generative adversarial nets (GANs) by \cite{gao2018robust} suggests that it is possible to compute depth-like robust estimators using similar techniques that optimize GANs. In this paper, we introduce a general \textit{learning via classification} framework based on the notion of proper scoring rules. This framework allows us to understand both matrix depth function and various  GANs through the lens of variational approximations of $f$-divergences induced by proper scoring rules. We then propose a new class of robust scatter estimators in this framework by carefully constructing discriminators with appropriate neural network structures. These estimators are proved to achieve the minimax rate of scatter estimation under Huber's contamination model. Our numerical results demonstrate its good performance under various settings against competitors in the literature.

\smallskip

\end{abstract}


\section{Introduction}

We study robust covariance matrix estimation under Huber's contamination model \citep{huber1964robust,huber1965robust}. In this setting, one has observations $X_1,...,X_n\stackrel{iid}{\sim} (1-\epsilon)N(0,\Sigma) + \epsilon Q$ in $\mathbb{R}^p$, and the goal is to estimate the covariance matrix $\Sigma$ using contaminated data without any assumption on the contamination distribution $Q$. Even though many robust covariance matrix estimators have been proposed and analyzed in the literature \citep{maronna1976robust,tyler1987distribution,zuo2000nonparametric,han2013optimal,mitra2014multivariate,wegkamp2016adaptive}, the problem of optimal covariance estimation under the contamination model has not been investigated until the recent work by \cite{chen2018robust}. It was shown in \cite{chen2018robust} that the minimax rate with respect to the squared operator norm $\opnorm{\wh{\Sigma}-\Sigma}^2$ is $\frac{p}{n}\vee\epsilon^2$. An important feature of the minimax rate is its \textit{dimension-free} dependence on the contamination proportion $\epsilon$ through the second term $\epsilon^2$. An estimator that can achieve the minimax rate is given by the maximizer of the covariance matrix depth function \citep{zhang2002some,chen2018robust,paindaveine2018halfspace}.

Despite its statistical optimality, the robust covariance matrix estimator that maximizes the depth function cannot be efficiently computed unless the dimension of the data is extremely low. This is the same weakness that is also shared by Tukey's halfspace depth \citep{tukey1975mathematics} and Rousseeuw and Hubert's regression depth \citep{rousseeuw1999regression}. In fact, even an approximate algorithm that computes these depth functions takes $O(e^{Cp})$ in time \citep{rousseeuw1998computing,amenta2000regression,chan2004optimal,chen2018robust}.

On the other hand, a recent connection between depth functions and \textit{Generative Adversarial Nets} (GANs) was discovered by \cite{gao2018robust}. The GAN \citep{goodfellow2014generative} is a very popular technique in deep learning to learn complex distributions such as the generating process of images. In the formulation of GAN, there is a \textit{generator} and a $\textit{discriminator}$. The generator, modeled by a neural network, is trying to learn a distribution as close to the data as possible, while the discriminator, modeled by another neural network, is trying to distinguish samples from the generator and data. This two-player game will reach its equilibrium when the discriminator cannot tell the difference between samples from the generator and the data, and that means the generator has successfully learned the underlying distribution of the data. Since GAN can be written as a minimax optimization problem, this suggests a mathematical resemblance to the robust estimators that are maximizers of depth functions, which are maximin optimization problems. Indeed, under the framework of $f$-Learning, it was shown by \cite{gao2018robust} that both procedures are minimizers of variational lower bounds of $f$-divergence functions. While GAN minimizes the Jensen-Shannon divergence, the robust estimators induced by depth functions all minimize the total variation distance. An alternative perturbation view on the connection between GAN and robust estimation was later discussed by \cite{zhu2019deconstructing}.

The connection between GAN, or more generally, $f$-GAN \citep{nowozin2016f}, and robust estimation opens a door of approximating these hard-to-compute depth functions by neural networks, and then standard techniques used to train GANs on a daily routine can be applied to compute various robust estimators. Appropriate choices of neural network structures have been discussed in \cite{gao2018robust}, but only for optimal robust location estimation.

In this paper, our goal is to understand appropriate network structures for optimal robust covariance or scatter matrix estimation under the framework of \textit{learning via classification}. Our main result shows that the network structures for optimal robust location estimation may not have sufficient discriminative power for optimal covariance matrix estimation. Therefore, we propose necessary modifications of the network structures so that optimal covariance matrix estimation under Huber's contamination model can be achieved.

The idea of \textit{learning via classification} has longstanding roots in statistics and machine learning \citep{diggle1984monte,freund1996game,buja2005loss,hyvarinen2005estimation,gneiting2007strictly,mohamed2016learning,dawid2007geometry,devroye2012combinatorial,sutherland2016generative,arjovsky2017wasserstein,gutmann2018likelihood,binkowski2018demystifying,baraud2018rho}. In this paper, we further expand this scope of statistcal learning by building a general estimation framework using classification with cost functions derived from proper scoring rules \citep{buja2005loss,gneiting2007strictly,dawid2007geometry}. Our framework is partly inspired by the discussion in \cite{mohamed2016learning}. Using Savage representation \citep{savage1971elicitation}, we identify a class of smooth objective functions that can be used for training optimal robust procedures under Huber's contamination model. The variational lower bounds of these objective functions cover important special cases including GANs and depth functions. 

\paragraph{Main Contributions.}

We summarize our main contributions of the paper as follows.

\begin{itemize}
\item We formulate a general framework of learning via classification using the concept of proper scoring rules. We show that the class of learning procedures under this framework has a one-to-one correspondence to a class of symmetric $f$-divergences, thanks to the Savage representation of the proper scoring rules. As a result, it leads to various forms of GANs and depth functions that are suitable for robust estimation under Huber's contamination model. 
\item We propose appropriate neural network structures for the classifiers in the proper scoring rules in the context of optimal robust covariance matrix estimation. We show that depending on whether the intercept node is included or not, the neural network is required to have at least one or two hidden layers for the robust covariance estimation task.
\item We also study robust scatter matrix estimation under general elliptical distributions. We show that in such a \textit{semiparametric learning} setting, one does not need to use a more complicated \textit{discriminator}, and only the \textit{generator} of the GAN needs to be modified.
\end{itemize}

\paragraph{Connections to the Literature.}

Our work is closely related to the recent developments on the statistical properties of GANs and the literature of robust covariance estimation under Huber's contamination model. For example, generalization bounds of GANs were established by \cite{zhang2017discrimination}. Nonparametric density estimation using GANs was studied by \cite{liang2017well}. Provable guarantees of learning Gaussian distributions with quadratic discriminators were established by \cite{feizi2017understanding}. Theoretical guarantees of learning Gaussian mixtures, exponential families and injective neural network generators were obtained by \cite{bai2018approximability}. The connection between GANs and robust estimation was established by \cite{gao2018robust} and also studied by \cite{zhu2019deconstructing}.
Polynomial-time algorithms for robust covariance matrix estimation under Huber's contamination model have been considered by \cite{lai2016agnostic,diakonikolas2016robust,du2017computationally,diakonikolas2017statistical} among others in the literature, with the main focus on error bounds in Frobenius norm and total variation loss.

\paragraph{Paper Organization.}

The rest of the paper is organized as follows. We develop an estimation framework of proper scoring rules in Section \ref{sec:psr}. In Section \ref{sec:nn}, we discuss robust covariance matrix estimation under this framework, and propose appropriate neural network classes for this task. An extension to simultaneous mean and covariance estimation is considered in Section \ref{sec:mean-cov}. In Section \ref{sect:elliptical}, we consider general robust scatter matrix estimation under elliptical distributions. Our numerical results are given in Section \ref{sec:num}. Section \ref{sec:proof} collects all the proofs in the paper.

\paragraph{Notation.}

We close this section by introducing the notation used in the paper. For $a,b\in\mathbb{R}$, let $a\vee b=\max(a,b)$ and $a\wedge b=\min(a,b)$. For an integer $m$, $[m]$ denotes the set $\{1,2,...,m\}$. Given a set $S$, $|S|$ denotes its cardinality, and $\mathbb{I}_S$ is the associated indicator function. For two positive sequences $\{a_n\}$ and $\{b_n\}$, the relation $a_n\lesssim b_n$ means that $a_n\leq Cb_n$ for some constant $C>0$, and $a_n\asymp b_n$ if both $a_n\lesssim b_n$ and $b_n\lesssim a_n$ hold. For a vector $v\in\mathbb{R}^p$, $\norm{v}$ denotes the $\ell_2$ norm, $\|v\|_{\infty}$ the $\ell_{\infty}$ norm, and $\|v\|_1$ the $\ell_1$ norm. For a matrix $A\in\mathbb{R}^{d_1\times d_2}$, we use $\opnorm{A}$ to denote its operator norm, which is its largest singular value. We use $\mathbb{P}$ and $\mathbb{E}$ to denote generic probability and expectation whose distribution is determined from the context. For two probability distributions $P_1$ and $P_2$, their total variation distance is $\TV(P_1,P_2)=\sup_B|P_1(B)-P_2(B)|$.
The sigmoid function, the ramp function, and the rectified linear unit function (ReLU) are denoted by $\sig(x)=\frac{1}{1+e^{-x}}$, $\ramp(x)=\max(\min(x+1/2,1),0)$, and $\relu(x)=\max(x,0)$.

\section{An Estimation Framework of Proper Scoring Rules}\label{sec:psr}

The idea of learning via classification can be formulated as a two-player game. Given a probability distribution $P$, two players act with strategies $T\in\mathcal{T}$ and $Q\in\mathcal{Q}$ to optimize the some cost function $D_T(P,Q)$. The first player chooses a classification rule $T$ in the class $\mathcal{T}$ to distinguish samples generated by $P$ from samples generated by $Q$. The second player then chooses a probability distribution $Q$ in the class $\mathcal{Q}$ so that samples generated by $Q$ cannot be distinguished from samples generated by $P$ even when the first player uses the optimal classification rule. This minimax game can be formulated as
\begin{equation}
\min_{Q\in\mathcal{Q}}\max_{T\in\mathcal{T}}D_T(P,Q). \label{eq:minimax-game}
\end{equation}
The minimax strategy for the second player $\argmin_{Q\in\mathcal{Q}}\max_{T\in\mathcal{T}}D_T(P,Q)$ can then be used to learn the distribution $P$. This is the principle behind the idea of GANs \citep{goodfellow2014generative,nowozin2016f} and many other statistical learning procedures in the literature \citep{diggle1984monte,freund1996game,buja2005loss,hyvarinen2005estimation,gneiting2007strictly,mohamed2016learning,dawid2007geometry,devroye2012combinatorial,sutherland2016generative,arjovsky2017wasserstein,gutmann2018likelihood,binkowski2018demystifying,baraud2018rho}. In this section, we discuss a class of cost functions $D_T(P,Q)$ induced by proper scoring rules. We show that the divergence function $D_{\mathcal{T}}(P,Q)=\max_{T\in\mathcal{T}}D_T(P,Q)$ can be viewed as a variational lower bound of some $f$-divergence, and the minimax strategy can be used as a robust estimator under Huber's contamination model.

\subsection{Proper Scoring Rules}

Consider a binary event space $\Omega=\{0,1\}$. A probabilistic forecast is a quoted probability $t\in[0,1]$ for either $0$ or $1$ to occur. A scoring rule $S$ is defined as a pair of functions $S(\cdot,1)$ and $S(\cdot,0)$. To be specific, $S(t,1)$ is the forecaster's reward if he or she quotes $t$ when the event $1$ occurs, and $S(t,0)$ is the reward when the event $0$ occurs.

Suppose the event occurs with probability $p$. Then, the expected reward for the forecaster is given by the formula
$$S(t;p)=pS(t,1) + (1-p)S(t,0).$$
It is called a proper scoring rule if
$$S(p;p)\geq S(t;p),\quad\text{for any }t\in[0,1],$$
or equivalently $p\in\argmax_{t\in[0,1]}S(t;p)$. The scoring rule is strictly proper when the equality above holds if and only if $t=p$. In this paper, we restrict our discussion to binary proper scoring rules. Readers interested in more general definitions are referred to \cite{buja2005loss,gneiting2007strictly,dawid2007geometry}.

\subsection{Savage Representation}

A scoring rule $S$ is \textit{regular} if both $S(\cdot,0)$ and $S(\cdot,1)$ are real-valued, except possibly that $S(0,1)=-\infty$ or $S(1,0)=-\infty$. The celebrated Savage representation \citep{savage1971elicitation} asserts that a regular scoring rule $S$ is proper if and only if there is a convex function $G(\cdot)$, such that
\begin{equation}
\begin{cases}
S(t,1) = G(t) + (1-t)G'(t), \\
S(t,0) = G(t) - tG'(t).
\end{cases}\label{eq:savage-rep}
\end{equation}
Here, $G'(t)$ is a subgradient of $G$ at the point $t$. Moreover, the statement also holds for strictly proper scoring rules when convex is replaced by strictly convex.

For any regular scoring rule, the convex function $G(\cdot)$ can be determined by
$$G(t)=S(t;t)=tS(t,1)+(1-t)S(t,0),$$
and Savage representation simply says that $S(t;t)$ is a convex function in $t$.

\subsection{Relation to $f$-Divergence}

Given two probability distributions $P$ and $Q$, a divergence function $D(P,Q)$ measures the difference between $P$ and $Q$. It satisfies the following two properties:
\begin{enumerate}
\item For any $P$ and $Q$, $D(P,Q)\geq 0$.
\item Whenever $P=Q$, $D(P,Q)=0$.
\end{enumerate}
Following the principle outlined in \cite{mohamed2016learning}, we show that a general class of divergence functions can be induced from proper scoring rules. To motivate the derivation, we consider a classification problem by introducing a binary latent variable $y\in\{0,1\}$. The conditional distribution of $X$ given $y$ is specified as $X|(y=1)\sim P$ and $X|(y=0)\sim Q$. We also assume that $\mathbb{P}(y=1)=\frac{1}{2}$ so that the joint distribution $(X,y)$ is fully specified. The classification problem is to find a function $T(X)\in[0,1]$ that forecasts the probability of $y=1$ given $X$. With a proper scoring rule $\{S(\cdot,1), S(\cdot,0)\}$, it is natural to consider the following cost function for the task,
\begin{eqnarray*}
&& \mathbb{E}\left[yS(T(X),1) + (1-y)S(T(X),0)\right] \\
&=& \frac{1}{2}\mathbb{E}_{X\sim P}S(T(X),1) + \frac{1}{2}\mathbb{E}_{X\sim Q}S(T(X),0).
\end{eqnarray*}
Then, one can find a good classification rule $T(\cdot)$ by maximizing the above objective over $T\in\mathcal{T}$. This leads to the following definition of a divergence function,
\begin{equation}
D_{\mathcal{T}}(P,Q)=\max_{T\in\mathcal{T}}\left[\frac{1}{2}\mathbb{E}_{X\sim P}S(T(X),1) + \frac{1}{2}\mathbb{E}_{X\sim Q}S(T(X),0)\right]-G(1/2),\label{eq:divergence-def}
\end{equation}
where $G(\cdot)$ is the convex function in the Savage representation of the proper scoring rule.

The definition (\ref{eq:divergence-def}) can be understood as a variational lower bound of some $f$-divergence. Given a convex function $f(\cdot)$ that satisfies $f(1)=0$, recall that the definition of the $f$-divergence between $P$ and $Q$ is given by
$$D_f(P\|Q)=\int f\left(\frac{dP}{dQ}\right)dQ.$$
\begin{proposition}\label{prop:div-PSR}
Given any regular proper scoring rule $\{S(\cdot,1), S(\cdot,0)\}$ and any class $\mathcal{T}\ni\left\{\frac{1}{2}\right\}$, $D_{\mathcal{T}}(P,Q)$ is a divergence function, and
\begin{equation}
D_{\mathcal{T}}(P,Q)\leq 
 D_f\Big(P\Big\|\frac{1}{2}P+\frac{1}{2}Q\Big),\label{variational-f}
\end{equation}
 where $f(t)=G(t/2)-G(1/2)$. Moreover, whenever $\mathcal{T}\ni \frac{dP}{dP+dQ}$, the inequality above becomes an equality.
\end{proposition}
\begin{proof}
Suppose $\mathcal{T}\ni\left\{\frac{1}{2}\right\}$, then $D_{\mathcal{T}}(P,Q)\geq \frac{1}{2}S(1/2,1)+\frac{1}{2}S(1/2,0)-G(1/2)=0$. When $P=Q$, we have $D_{\mathcal{T}}(P,Q)\leq \max_{t\in[0,1]}\left[G(t)-G(1/2)-(t-1/2)G'(t)\right]\leq 0$ by the convexity of $G(\cdot)$, and therefore $D_{\mathcal{T}}(P,Q)=0$, which implies it is a divergence function. Since $\{S(\cdot,1), S(\cdot,0)\}$ is a proper scoring rule, $p(x)S(T(x),1)+q(x)S(T(x),0)$ is maximized at $T(x)=\frac{p(x)}{p(x)+q(x)}$. Thus,
\begin{eqnarray*}
D_{\mathcal{T}}(P,Q) &\leq& \frac{1}{2}\mathbb{E}_{X\sim P}S\left(\frac{dP}{dP+dQ}(X),1\right) + \frac{1}{2}\mathbb{E}_{X\sim Q}S\left(\frac{dP}{dP+dQ}(X),0\right)-G(1/2) \\
&=& \frac{1}{2}\mathbb{E}_{X\sim P}G\left(\frac{dP}{dP+dQ}(X)\right) + \frac{1}{2}\mathbb{E}_{X\sim Q}G\left(\frac{dP}{dP+dQ}(X)\right)-G(1/2)\\
&=& D_f\Big(P\Big\|\frac{1}{2}P+\frac{1}{2}Q\Big),
\end{eqnarray*}
and obviously the inequality above becomes an equality when $\mathcal{T}\ni \frac{dP}{dP+dQ}$.
\end{proof}

It is worth noting that $D_f\Big(P\Big\|\frac{1}{2}P+\frac{1}{2}Q\Big)$ is in general not symmetric with respect to $P$ and $Q$. However, when the regular proper scoring rule is symmetric in the sense that $S(t,1)=S(1-t,0)$, we have $G(t)=G(1-t)$, or equivalently, $f(t)=f(2-t)$, in which case the corresponding $f$-divergence satisfies
$$D_f\Big(P\Big\|\frac{1}{2}P+\frac{1}{2}Q\Big)=D_f\Big(Q\Big\|\frac{1}{2}P+\frac{1}{2}Q\Big),$$
and is symmetric.

\subsection{Variational Lower Bounds and GANs}

The variational form of the divergence function makes it easy to define a sample version of (\ref{eq:divergence-def}). Replacing $\mathbb{E}_{X\sim P}$ in (\ref{eq:divergence-def}) by the empirical measure, we have a divergence function between $\frac{1}{n}\sum_{i=1}^n\delta_{X_i}$ and $Q$, which is a useful objective function for statistical estimation. Given a class of probability measures $\mathcal{Q}$, the induced estimator of $P$ is defined by
\begin{equation}
\wh{P} = \argmin_{Q\in\mathcal{Q}}\max_{T\in\mathcal{T}}\left[\frac{1}{n}\sum_{i=1}^n S(T(X_i),1) + \mathbb{E}_{X\sim Q}S(T(X),0)\right]. \label{eq:GAN}
\end{equation}
We drop the term $-G(1/2)$ in (\ref{eq:divergence-def}) because it is a constant that does not affect the definition of (\ref{eq:GAN}).
The formula (\ref{eq:GAN}) has an interpretation of a minimax game between two players. The goal of the first player is to find the best discriminator $T$ that learns whether a sample is from the empirical distribution or the model distribution $Q$. The second player is to find a model distribution $Q$ as close to the empirical distribution as possible so that the first player cannot tell the difference. In the context of deep learning, both the discriminator class $\mathcal{T}$ and the generator class $\mathcal{Q}$ are modeled by neural networks, and (\ref{eq:GAN}) is recognized as the technique of \textit{Generative Adversarial Nets} proposed by \cite{goodfellow2014generative}. The relation between GANs and proper scoring rules was discussed by \cite{mohamed2016learning} in the context of learning implicit models.

\subsection{Examples}\label{subsect:sc-example}

\begin{enumerate}
\item \textit{Log Score.} The log score is perhaps the most commonly used rule because of its various intriguing properties \citep{jiao2015justification}. The scoring rule with $S(t,1)=\log t$ and $S(t,0)=\log(1-t)$ is regular and strictly proper. Its Savage representation is given by the convex function $G(t)=t\log t+(1-t)\log(1-t)$, which is interpreted as the negative Shannon entropy of $\text{Bernoulli}(t)$. The corresponding divergence function $D_{\mathcal{T}}(P,Q)$, according to Proposition \ref{prop:div-PSR}, is a variational lower bound of the Jensen-Shannon divergence
$$\JS(P,Q)=\frac{1}{2}\int\log\left(\frac{dP}{dP+dQ}\right)dP+\frac{1}{2}\int\log\left(\frac{dQ}{dP+dQ}\right)dQ+\log 2.$$
Its sample version (\ref{eq:GAN}) is the original GAN proposed by \cite{goodfellow2014generative} that is widely used in learning distributions of images.
\item \textit{Zero-One Score.} The zero-one score $S(t,1)=2\mathbb{I}\{t\geq 1/2\}$ and $S(t,0)=2\mathbb{I}\{t<1/2\}$ is also known as the misclassification loss. This is a regular proper scoring rule but not strictly proper. The induced divergence function $D_{\mathcal{T}}(P,Q)$ is a variational lower bound of the total variation distance
$$\TV(P,Q)=P\left(\frac{dP}{dQ}\geq 1\right)-Q\left(\frac{dP}{dQ}\geq 1\right)=\frac{1}{2}\int |dP-dQ|.$$
The sample version (\ref{eq:GAN}) is recognized as the TV-GAN that was extensively studied by \cite{gao2018robust} in the context of robust estimation.
\item \textit{Quadratic Score.} Also known as the Brier score \citep{brier1950verification}, the definition is given by $S(t,1)=-(1-t)^2$ and $S(t,0)=-t^2$. The corresponding convex function in the Savage representation is given by $G(t)=-t(1-t)$. By Proposition \ref{prop:div-PSR}, the divergence function (\ref{eq:divergence-def}) induced by this regular strictly proper scoring rule is a variational lower bound of the following divergence function,
$$\Delta(P,Q)=\frac{1}{8}\int\frac{(dP-dQ)^2}{dP+dQ},$$
known as the triangular discrimination. The sample version (\ref{eq:GAN}) belongs to the family of least-squares GANs proposed by \cite{mao2017least}.
\item \textit{Boosting Score.} The boosting score was introduced by \cite{buja2005loss} with $S(t,1)=-\left(\frac{1-t}{t}\right)^{1/2}$ and $S(t,0)=-\left(\frac{t}{1-t}\right)^{1/2}$ and has an connection to the AdaBoost algorithm. The corresponding convex function in the Savage representation is given by $G(t)=-2\sqrt{t(1-t)}$. The induced divergence function $D_{\mathcal{T}}(P,Q)$ is thus a variational lower bound of the squared Hellinger distance
$$H^2(P,Q)=\frac{1}{2}\int\left(\sqrt{dP}-\sqrt{dQ}\right)^2.$$
\item \textit{Beta Score.} A general Beta family of proper scoring rules was introduced by \cite{buja2005loss} with $S(t,1)=-\int_t^1c^{\alpha-1}(1-c)^{\beta}dc$ and $S(t,0)=-\int_0^tc^{\alpha}(1-c)^{\beta-1}dc$ for any $\alpha,\beta>-1$. The log score, the quadratic score and the boosting score are special cases of the Beta score with $\alpha=\beta=0$, $\alpha=\beta=1$, $\alpha=\beta=-1/2$. The zero-one score is a limiting case of the Beta score by letting $\alpha=\beta\rightarrow\infty$. Moreover, it also leads to asymmetric scoring rules with $\alpha\neq \beta$.
\end{enumerate}

\subsection{TV-GAN and The Matrix Depth Function}

With the zero-one loss, (\ref{eq:GAN}) is specialized as
\begin{equation}
\wh{P}=\argmin_{Q\in\mathcal{Q}}\max_{T\in\mathcal{T}}\left[\frac{1}{n}\sum_{i=1}^n\mathbb{I}\{T(X_i)\geq 1/2\}+\mathbb{E}_{X\sim Q}\mathbb{I}\{T(X)<1/2\}\right].\label{eq:TV-GAN}
\end{equation}
We also consider a variation of (\ref{eq:TV-GAN}) defined by
\begin{equation}
\wh{P}=\argmin_{Q\in\mathcal{Q}}\max_{T\in\mathcal{T}_Q}\left[\frac{1}{n}\sum_{i=1}^n\mathbb{I}\{T(X_i)\geq 1/2\}+\mathbb{E}_{X\sim Q}\mathbb{I}\{T(X)<1/2\}\right]. \label{eq:TV-GAN-local}
\end{equation}
The subtle difference of (\ref{eq:TV-GAN-local}) compared with (\ref{eq:TV-GAN}) is the dependence of the discriminator class on $Q$. In fact, both (\ref{eq:TV-GAN-local}) and (\ref{eq:TV-GAN}) can be regarded as the minimizers of variational lower bounds of the total variation distance.
The connection between (\ref{eq:TV-GAN-local}) and various depth functions in robust estimation was discussed in an $f$-Learning framework by \cite{gao2018robust}.

For the purpose of covariance matrix estimation, we show that (\ref{eq:TV-GAN-local}) leads to the definition of the matrix depth function \citep{zhang2002some,chen2018robust,paindaveine2018halfspace}. Let $\mathcal{E}_p$ be the set of all $p\times p$ covariance matrices. We set
$$\mathcal{Q}=\left\{N(0,\Gamma): \Gamma\in\mathcal{E}_p\right\}.$$
and
$$\mathcal{T}_{N(0,\Gamma)}=\left\{T=\frac{dN(0,\beta\wt{\Gamma})}{dN(0,\beta\wt{\Gamma})+dN(0,\beta\Gamma)}: \wt{\Gamma}^{-1}=\Gamma^{-1}+\wt{r}uu^T\in\mathcal{E}_p, |\wt{r}|\leq r, \|u\|=1\right\},$$
where $\beta$ is a scalar determined by the equation $\mathbb{P}(N(0,1)\leq\sqrt{\beta})=3/4$.
The choice of the local discriminator class $\mathcal{T}_{N(0,\Gamma)}$ is motivated by the conclusion of Proposition \ref{prop:div-PSR} that the optimal discriminator between $P$ and $Q$ is $\frac{dP}{dP+dQ}$. However, there are two distinguished features. The first one is that the discriminator uses a slightly different form from the optimal one by including a multiplicative factor $\beta$, in order to adjust the ratio between standard deviation and median absolute deviation. The second one is that $\wt{\Gamma}^{-1}$ only ranges in a rank-one neighborhood of $\Gamma^{-1}$ to ensure Fisher consistency.

By direct calculation, we have
\begin{equation}
\mathbb{I}\left\{\frac{dN(0,\beta\wt{\Gamma})}{dN(0,\beta\wt{\Gamma})+dN(0,\beta\Gamma)}(X)\geq 1/2\right\}=\mathbb{I}\left\{\wt{r}|u^TX|^2\leq \log(1+\beta\wt{r}u^T\Gamma u)\right\}. \label{eq:matrix-depth-event}
\end{equation}
Therefore, we can write (\ref{eq:TV-GAN-local}) as
\begin{eqnarray}
\nonumber \wh{\Sigma} &=& \argmin_{\Gamma\in\mathcal{E}_p}\max_{\substack{\|u\|=1\\|\wt{r}|\leq r}}\Bigg[\frac{1}{n}\sum_{i=1}^n\mathbb{I}\left\{\wt{r}|u^TX_i|^2\leq \log(1+\beta\wt{r}u^T\Gamma u)\right\} \\
\label{eq:matrix-depth-r} && \qquad\qquad\qquad\qquad- \mathbb{P}_{X\sim N(0,\Gamma)}\left\{\wt{r}|u^TX|^2\leq \log(1+\beta\wt{r}u^T\Gamma u)\right\}\Bigg].
\end{eqnarray}
Since $\lim_{\wt{r}\rightarrow 0}\frac{\log (1+\beta\wt{r}u^T\Gamma u)}{\wt{r}u^T\Gamma u}=\beta$, the limiting event of (\ref{eq:matrix-depth-event}) is either $\mathbb{I}\{|u^TX|^2\leq \beta u^T\Gamma u\}$ or $\mathbb{I}\{|u^TX|^2\geq \beta u^T\Gamma u\}$, depending on whether $\wt{r}$ tends to zero from left or from right. Moreover, $\mathbb{P}_{X\sim N(0,\Gamma)}\{|u^TX|^2\leq \beta u^T\Gamma u\}=1/2$ by the definition of $\beta$. Therefore, as $r\rightarrow 0$, the formula (\ref{eq:matrix-depth-r}) becomes
\begin{equation}
\wh{\Sigma} = \argmin_{\Gamma\in\mathcal{E}_p}\max_{\|u\|=1}\left[\frac{1}{n}\sum_{i=1}^n\mathbb{I}\{|u^TX_i|^2\leq \beta u^T\Gamma u\}\vee \frac{1}{n}\sum_{i=1}^n\mathbb{I}\{|u^TX_i|^2\geq \beta u^T\Gamma u\}\right], \label{eq:matrix-depth}
\end{equation}
which recovers the definition of the matrix depth function.

\section{Network Structures for Robust Covariance Matrix Estimation}\label{sec:nn}

The main goal of the paper is to study the property of the estimator (\ref{eq:GAN}) in the context of robust covariance matrix estimation. Define
$$\mathcal{E}_p(M)=\left\{\Gamma\in\mathcal{E}_p:\opnorm{\Gamma}\leq M\right\}.$$
For covariance matrix estimation, we write (\ref{eq:GAN}) as
\begin{equation}
\wh{\Sigma} = \argmin_{\Gamma\in\mathcal{E}_p(M)}\max_{T\in\mathcal{T}}\left[\frac{1}{n}\sum_{i=1}^n S(T(X_i),1) + \mathbb{E}_{X\sim N(0,\Gamma)}S(T(X),0)\right]. \label{eq:covariance-GAN}
\end{equation}
This estimator extends the one induced by the matrix depth function (\ref{eq:matrix-depth}) to general proper scoring rules.

We consider i.i.d. observations drawn from Huber's $\epsilon$-contamination model \citep{huber1964robust,huber1965robust}. That is,
\begin{equation}
X_1,...,X_n \stackrel{iid}{\sim} (1-\epsilon)N(0,\Sigma) + \epsilon Q. \label{eq:Huber-con}
\end{equation}
In other words, each observation has an $\epsilon$ probability to be drawn from an unknown contamination distribution $Q$. A more general data generating process is called the \textit{strong contamination model}. In such a setting, we have
\begin{equation}
X_1,...,X_n \stackrel{iid}{\sim} P\quad\text{for some }P\text{ satisfying }\TV(P,N(0,\Sigma))\leq\epsilon, \label{eq:strong-con}
\end{equation}
which means that the observations are drawn from an unknown probability distribution in a total variation neighborhood of $N(0,\Sigma)$. It is easy to see that (\ref{eq:Huber-con}) implies (\ref{eq:strong-con}) so that (\ref{eq:strong-con}) is a more general notion of contamination. While the contamination is only allowed to be added into good samples in (\ref{eq:Huber-con}), the adversarial can now choose to remove some good samples after looking at the data in the setting of (\ref{eq:strong-con}). See \cite{diakonikolas2016robust} for a detailed discussion on various contamination models.

Under both (\ref{eq:Huber-con}) and (\ref{eq:strong-con}), the minimax rate of covariance matrix estimation with respect to the loss function $\opnorm{\wh{\Sigma}-\Sigma}^2$ is $\frac{p}{n}\vee\epsilon^2$, and can be achieved by (\ref{eq:matrix-depth}). This fact was proved by \cite{chen2018robust} under Huber's contamination model (\ref{eq:Huber-con}), and the same proof can be extended to the strong contamination model (\ref{eq:strong-con}).

Despite the statistical optimality of the estimator (\ref{eq:matrix-depth}), its optimization is computationally infeasible in practice whenever the dimension exceeds $10$ \citep{chen2018robust}. This is partly due to the fact that the zero-one loss is not smooth. However, even the smooth version of the TV-GAN was shown to be computationally intractable, which motivates \cite{gao2018robust} to consider alternative proper scoring rules such as the log score in the setting of robust mean estimation.

In this section, we study the statistical properties of (\ref{eq:covariance-GAN}) with general proper scoring rules. We will discuss appropriate choices of the discriminator class $\mathcal{T}$ for robust covariance matrix estimation. To leverage the computational strategies developed in the deep learning community \citep{goodfellow2014generative,radford2015unsupervised,salimans2016improved}, we consider $\mathcal{T}$ that is some family of neural network classifiers. Then, the structure of the neural nets is essential in determining the statistical properties of (\ref{eq:covariance-GAN}). We will present two network structures that are not appropriate for robust covariance matrix estimation, and then show simple modifications of the two structures lead to optimal estimation.

\subsection{Inappropriate Network Structures}

Consider the following two discriminator classes. The first class contains two-layer sigmoid neural nets,
\begin{equation}
\mathcal{T}_1 = \left\{T(x)=\sig\left(\sum_{j\geq 1}w_j\sig(u_j^Tx)\right): \sum_{j\geq 1}|w_j|\leq\kappa,u_j\in\mathbb{R}^p\right\}. \label{eq:net-structure-1}
\end{equation}
The second class also contains two-layer neural nets, but uses ReLU activations in the hidden layer, 
\begin{equation}
\mathcal{T}_2 = \left\{T(x)=\sig\left(\sum_{j\geq 1}w_j\relu(u_j^Tx)\right): \sum_{j\geq 1}|w_j|\leq\kappa,\|u_j\|\leq 1\right\}. \label{eq:net-structure-2}
\end{equation}
The network structures of $\mathcal{T}_1$ and $\mathcal{T}_2$ are visualized in Figure \ref{fig:T1T2}.
\begin{figure}[!ht]
\centering
\includegraphics[width=.70\textwidth]{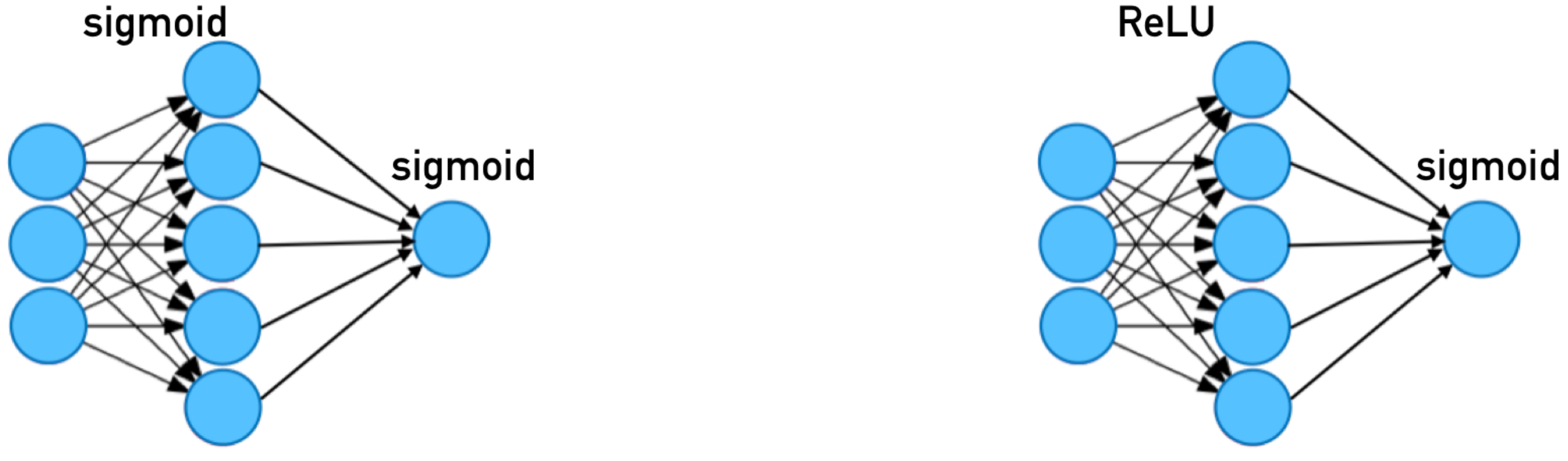}
\caption{Two structures of neural nets that are not suitable for robust covariance matrix estimation.}\label{fig:T1T2} 
\end{figure}
The reasons that they do not work are different for the two structures. To construct concrete counterexamples, we focus on the log score in this section. That is, we consider the estimator
\begin{equation}
\wh{\Sigma} = \argmin_{\Gamma\in\mathcal{E}_p(M)}\max_{T\in\mathcal{T}}\left[\frac{1}{n}\sum_{i=1}^n \log T(X_i) + \mathbb{E}_{X\sim N(0,\Gamma)}\log(1-T(X))\right]. \label{eq:covariance-GAN-log}
\end{equation}

The first class (\ref{eq:net-structure-1}) leads to optimal robust mean estimation, but fails to learn the covariance matrix even if there is no contamination in the data. The following result shows the capability of (\ref{eq:net-structure-1}) in learning a mean vector.
\begin{proposition}\label{prop:robust-mean}
Consider the estimator
$$\wh{\theta}=\argmin_{\eta\in\mathbb{R}^p}\max_{T\in\mathcal{T}_1}\left[\frac{1}{n}\sum_{i=1}^n \log T(X_i) + \mathbb{E}_{X\sim N(\eta,I_p)}\log(1-T(X))\right],$$
with $\mathcal{T}_1$ specified by (\ref{eq:net-structure-1}). Assume $\frac{p}{n}+\epsilon^2\leq c$ for some sufficiently small constant $c>0$, and set $\kappa=O\left(\sqrt{\frac{p}{n}}+\epsilon\right)$. With i.i.d. observations $X_1,...,X_n\sim P$, we have
$$\|\wh{\theta}-\theta\|^2\leq C\left(\frac{p}{n}\vee\epsilon^2\right),$$
with probability at least $1-e^{-C'(p+n\epsilon^2)}$ uniformly over all $\|\theta\|\leq M=O(1)$ and all $P$ such that $\TV(P,N(\theta,I_p))\leq\epsilon$. The constants $C,C'>0$ are universal.
\end{proposition}
The success of robust estimation via a two-layer neural network was first proved by \cite{gao2018robust} under Huber's $\epsilon$-contamination model. Proposition \ref{prop:robust-mean} extends the result to the strong contamination model. However, the same neural network structure cannot learn a covariance matrix, as is shown by the following result.
\begin{proposition}\label{prop:cov-flat}
With $\mathcal{T}_1$ specified by (\ref{eq:net-structure-1}), the function
$$F(\Sigma,\Gamma)=\max_{T\in\mathcal{T}_1}\left[\mathbb{E}_{X\sim N(0,\Sigma)}\log T(X) + \mathbb{E}_{X\sim N(0,\Gamma)}\log(1-T(X))\right]$$
is a constant for all $\Sigma,\Gamma\in\mathcal{E}_p$.
\end{proposition}
In an ideal situation where $\epsilon=0$ and $n=\infty$, the estimator (\ref{eq:covariance-GAN-log}) becomes $\argmin_{\Gamma}F(\Sigma,\Gamma)$. However, Proposition \ref{prop:cov-flat} shows that the objective function $F(\Sigma,\Gamma)$ is completely flat, and thus every $\Gamma$ is a global minimizer.

The second discriminator class (\ref{eq:net-structure-2}) has a different problem. It actually leads to optimal covariance matrix estimation when $\epsilon=0$, but it does not lead to robust estimation.
\begin{proposition}\label{prop:cov-no-con}
Consider the estimator (\ref{eq:covariance-GAN-log}), where $\mathcal{T}=\mathcal{T}_2$ is specified by (\ref{eq:net-structure-2}) with at least two units in the hidden layer. Assume $\frac{p}{n}\leq c$ for some sufficiently small constant $c>0$, and set $\kappa=O\left(\sqrt{\frac{p}{n}}\right)$. With i.i.d. observations $X_1,...,X_n\sim N(0,\Sigma)$, we have
$$\opnorm{\wh{\Sigma}-\Sigma}^2\leq C\frac{p}{n},$$
with probability at least $1-e^{-C'p}$ uniformly over all $\opnorm{\Sigma}\leq M=O(1)$. The constants $C,C'>0$ are universal.
\end{proposition}
The comparison between Proposition \ref{prop:cov-flat} and Proposition \ref{prop:cov-no-con} shows that the subtle difference between the activation functions in the hidden layer directly affects the consistency of covariance matrix estimation. A simple change from sigmoid to ReLU leads to an optimal error rate in Proposition \ref{prop:cov-no-con}. However, as long as there is contamination in the data, the structure does not lead to robust estimation. We show a one-dimensional counterexample in the following proposition.
\begin{proposition}\label{prop:T2-not-robust}
With $\mathcal{T}_2$ specified by (\ref{eq:net-structure-2}), we have
$$[(1-\epsilon)\sigma +\epsilon\tau]^2\in\argmin_{\gamma^2}\max_{T\in\mathcal{T}_2}\left[\mathbb{E}_{X\sim (1-\epsilon)N(0,\sigma^2)+\epsilon N(0,\tau^2)}\log T(X) + \mathbb{E}_{X\sim N(0,\gamma^2)}\log(1-T(X))\right],$$
for any $\sigma^2,\tau^2>0$.
\end{proposition}
The proposition considers a setting with an $\epsilon$ fraction of contaminated observations generated from $N(0,\tau^2)$. Then, in the ideal situation with $n=\infty$, $[(1-\epsilon)\sigma +\epsilon\tau]^2$ is global minimizer. Since the value of $\tau$ is not restricted, this global minimizer can be arbitrarily far away from the variance $\sigma^2$ of the good samples.

\subsection{Appropriate Network Structures}

The network structures (\ref{eq:net-structure-1}) and (\ref{eq:net-structure-2}) can both be slightly modified to achieve optimal robust covariance matrix estimation. For the first discriminator class (\ref{eq:net-structure-1}), we only need to add an intercept node in the input layer, which leads to the definition of the following discriminator class,
\begin{equation}
\mathcal{T}_3 = \left\{T(x)=\sig\left(\sum_{j\geq 1}w_j\sig(u_j^Tx+b_j)\right): \sum_{j\geq 1}|w_j|\leq\kappa,u_j\in\mathbb{R}^p,b_j\in\mathbb{R}\right\}. \label{eq:net-structure-3}
\end{equation}
For the second class (\ref{eq:net-structure-2}), we need to add an extra sigmoid hidden layer. This gives
\begin{eqnarray}
\nonumber \mathcal{T}_4 &=& \Bigg\{T(x) = \sig\left(\sum_{j\geq 1}w_j\sig\left(\sum_{l=1}^Hv_{jl}\relu(u_l^Tx)\right)\right): \\
\label{eq:net-structure-4} &&\qquad\qquad\qquad\qquad\qquad\qquad\qquad \sum_{j\geq 1}|w_j|\leq \kappa_1, \sum_{l=1}^H|v_{jl}|\leq\kappa_2, \|u_l\|\leq 1\Bigg\}.
\end{eqnarray}
\begin{figure}[!ht]
\centering
\includegraphics[width=.80\textwidth]{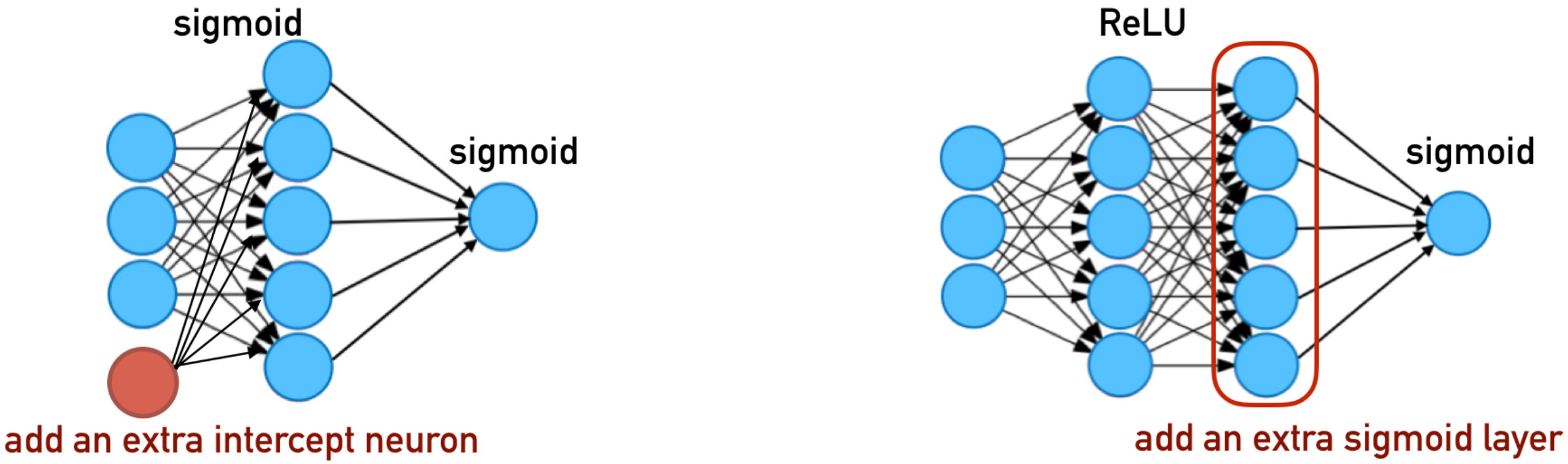}
\caption{Simple fixes of the two network structures (\ref{eq:net-structure-1}) and (\ref{eq:net-structure-2}).}\label{fig:T3T4} 
\end{figure}
These two modifications of (\ref{eq:net-structure-1}) and (\ref{eq:net-structure-2}) are illustrated in Figure \ref{fig:T3T4}.

We study the covariance matrix estimator (\ref{eq:covariance-GAN}) with a general regular proper scoring rule. Recall that a regular proper scoring rule admits the Savage representation (\ref{eq:savage-rep}) with a convex function $G(\cdot)$. We impose the following assumption on the convex function $G(\cdot)$.
\begin{condition}\label{cond:G}
We assume $G^{(2)}(1/2)>0$ and $G^{(3)}(t)$ is continuous at $t=1/2$. Moreover, there is a universal constant $c_0>0$, such that $2G^{(2)}(1/2)\geq G^{(3)}(1/2)+c_0$.
\end{condition}
Condition \ref{cond:G} implies the scoring rule $\{S(\cdot,1), S(\cdot,0)\}$ is induced by two smooth functions, which excludes the zero-one loss. This is fine, because the zero-one loss was already studied as the matrix depth function in \cite{chen2018robust}. This paper only focuses on scoring rules that are feasible to optimize, and thus it is sufficient to restrict our results to smooth ones. The condition $2G^{(2)}(1/2)\geq G^{(3)}(1/2)+c_0$ is automatically satisfied by a symmetric scoring rule, because $S(t,1)=S(1-t,0)$ immediately implies that $G^{(3)}(1/2)=0$. For the Beta score with $S(t,1)=-\int_t^1c^{\alpha-1}(1-c)^{\beta}dc$ and $S(t,0)=-\int_0^tc^{\alpha}(1-c)^{\beta-1}dc$ for any $\alpha,\beta>-1$, it is easy to check that such a $c_0$ (only depending on $\alpha,\beta$) exists as long as $|\alpha-\beta|<1$.

\begin{thm}\label{thm:cov-T3}
Consider the estimator (\ref{eq:covariance-GAN}) that is induced by a regular proper scoring rule that satisfies Condition \ref{cond:G}, and $\mathcal{T}=\mathcal{T}_3$ is specified by (\ref{eq:net-structure-3}). Assume $\frac{p}{n}+\epsilon^2\leq c$ for some sufficiently small constant $c>0$, and set $\kappa=O\left(\sqrt{\frac{p}{n}}+\epsilon\right)$. Then, under the data generating process (\ref{eq:strong-con}), we have
$$\opnorm{\wh{\Sigma}-\Sigma}^2\leq C\left(\frac{p}{n}\vee\epsilon^2\right),$$
with probability at least $1-e^{-C'(p+n\epsilon^2)}$ uniformly over all $\opnorm{\Sigma}\leq M=O(1)$. The constants $C,C'>0$ are universal.
\end{thm}

The theorem shows that the discriminator class (\ref{eq:net-structure-3}) leads to optimal robust covariance matrix estimation, while the only difference between (\ref{eq:net-structure-1}) and (\ref{eq:net-structure-3}) is the inclusion of the intercept neuron in the bottom layer of the network in the class (\ref{eq:net-structure-3}). In contrast to the common understanding that whether to include the intercept neuron in a neural network structure is only a matter of data normalization, here for the purpose of robust covariance matrix estimation using proper scoring rules, it is a fundamental issue.

\begin{thm}\label{thm:cov-T4}
Consider the estimator (\ref{eq:covariance-GAN}) that is induced by a regular proper scoring rule that satisfies Condition \ref{cond:G}, and $\mathcal{T}=\mathcal{T}_4$ is specified by (\ref{eq:net-structure-4}). Assume $\frac{p}{n}+\epsilon^2\leq c$ for some sufficiently small constant $c>0$. Set $H\geq 2$, $\kappa_1=O\left(\sqrt{\frac{p}{n}}+\epsilon\right)$, and $1\leq\kappa_2=O(1)$. Then, under the data generating process (\ref{eq:strong-con}), we have
$$\opnorm{\wh{\Sigma}-\Sigma}^2\leq C\left(\frac{p}{n}\vee\epsilon^2\right),$$
with probability at least $1-e^{-C'(p+n\epsilon^2)}$ uniformly over all $\opnorm{\Sigma}\leq M=O(1)$. The constants $C,C'>0$ are universal.
\end{thm}

The ReLU activation function is widely used in training deep neural network models because of its superior optimization properties \citep{glorot2011deep}. To estimate a covariance matrix, Theorem \ref{thm:cov-T4} shows that it can only be used after the two top layers. Otherwise, according to Proposition \ref{prop:T2-not-robust}, the estimator would not be robust against arbitrary contamination.

\section{Simultaneous Estimation of Mean and Covariance}\label{sec:mean-cov}

In this section, we consider a more general setting where the data generating process is
\begin{equation}
X_1,...,X_n \stackrel{iid}{\sim} P\quad\text{for some }P\text{ satisfying }\TV(P,N(\theta,\Sigma))\leq\epsilon. \label{eq:strong-con-mean-cov}
\end{equation}
That is, both the mean vector $\theta$ and the covariance matrix $\Sigma$ are unknown. Covariance matrix estimation with an unknown mean was considered in the literature. For example, modifications of the matrix depth function (\ref{eq:matrix-depth}) that incorporate the unknown mean were considered by \cite{chen2018robust,paindaveine2018halfspace}. In particular, a U-statistics version of the depth function (\ref{eq:matrix-depth}) was considered in \cite{chen2018robust}. The idea is to take advantage of the fact that $(X_i-X_j)/\sqrt{2}\sim N(0,\Sigma)$ for all pairs $i<j$. Applying the same modification to (\ref{eq:covariance-GAN}), we obtain
\begin{equation}
\wh{\Sigma} = \argmin_{\Gamma\in\mathcal{E}_p(M)}\max_{T\in\mathcal{T}}\left[\frac{1}{{n\choose 2}}\sum_{1\leq i<j\leq n} S(T((X_i-X_j)/\sqrt{2}),1) + \mathbb{E}_{X\sim N(0,\Gamma)}S(T(X),0)\right]. \label{eq:covariance-GAN-U}
\end{equation}
Theorem \ref{thm:cov-T3} and Theorem \ref{thm:cov-T4} can then be easily extended to the setting (\ref{eq:strong-con-mean-cov}) with an unknown mean by using the estimator (\ref{eq:covariance-GAN-U}).

In addition to the U-statistics version (\ref{eq:covariance-GAN-U}), we propose another modification of (\ref{eq:covariance-GAN}) that allows for simultaneous estimation of mean and covariance. The procedure is defined by
\begin{equation}
(\wh{\theta},\wh{\Sigma}) = \argmin_{\eta\in\mathbb{R}^p,\Gamma\in\mathcal{E}_p(M)}\max_{T\in\mathcal{T}}\left[\frac{1}{n}\sum_{i=1}^n S(T(X_i),1) + \mathbb{E}_{X\sim N(\eta,\Gamma)}S(T(X),0)\right]. \label{eq:mean-covariance-GAN}
\end{equation}
Note that the generator class is $\{N(\eta,\Gamma): \eta\in\mathbb{R}^p,\Gamma\in\mathcal{E}_p(M)\}$ compared with the centered class in (\ref{eq:covariance-GAN}).

We also introduce a more general discriminator class of deep neural nets. We first define a sigmoid bottom layer
$$\mathcal{G}_{\sig}=\left\{g(x)=\sig(u^Tx+b):u\in\mathbb{R}^p,b\in\mathbb{R}\right\}.$$
Then, with $\mathcal{G}^1(B)=\mathcal{G}_{\sig}$, we inductively define
$$\mathcal{G}^{l+1}(B)=\left\{g(x)=\relu\left(\sum_{h\geq 1}v_hg_h(x)\right):\sum_{h\geq 1}|v_h|\leq B, g_h\in\mathcal{G}^l(B)\right\}.$$
Note that the neighboring two layers are connected via ReLU activation functions.
Finally, the network structure is defined by
\begin{equation}
\mathcal{T}^L(\kappa,B) = \Bigg\{T(x)=\sig\left(\sum_{j\geq 1}w_jg_j(x)\right): \sum_{j\geq 1}|w_j|\leq \kappa, g_{j}\in\mathcal{G}^{L}(B)\Bigg\}.\label{eq:net-structure-deep}
\end{equation}
This is a neural network class that consists of $L$ hidden layers. When $L=1$, (\ref{eq:net-structure-deep}) recovers the definition of the class (\ref{eq:net-structure-3}).

\begin{thm}\label{thm:mean-cov}
Consider the estimator (\ref{eq:mean-covariance-GAN}) that is induced by a regular proper scoring rule that satisfies Condition \ref{cond:G}. The discriminator class $\mathcal{T}=\mathcal{T}^L(\kappa,B)$ is specified by (\ref{eq:net-structure-deep}). Assume $\frac{p}{n}+\epsilon^2\leq c$ for some sufficiently small constant $c>0$. Set $1\leq L=O(1)$, $1\leq B=O(1)$, and $\kappa=O\left(\sqrt{\frac{p}{n}}+\epsilon\right)$. Then, under the data generating process (\ref{eq:strong-con-mean-cov}), we have
\begin{eqnarray*}
\|\wh{\theta}-\theta\|^2 &\leq& C\left(\frac{p}{n}\vee\epsilon^2\right), \\
\opnorm{\wh{\Sigma}-\Sigma}^2 &\leq& C\left(\frac{p}{n}\vee\epsilon^2\right),
\end{eqnarray*}
with probability at least $1-e^{-C'(p+n\epsilon^2)}$ uniformly over all $\theta\in\mathbb{R}^p$ and all $\opnorm{\Sigma}\leq M=O(1)$. The constants $C,C'>0$ are universal.
\end{thm}

\section{Elliptical Distributions}\label{sect:elliptical}

One of the most important statistical properties of the depth-based estimator (\ref{eq:matrix-depth}) is its ability to adapt to general elliptical distributions \citep{chen2018robust}. In this section, we show that the same property can also be achieved by robust estimators induced by proper scoring rules.
\begin{definition}[\cite{fang2017symmetric}]
\label{def:elliptical}
A random vector $X\in\mathbb{R}^p$ follows an elliptical distribution if and only if it has the representation $X=\theta+\xi AU$, where $\theta\in\mathbb{R}^p$ and $A\in\mathbb{R}^{p\times r}$ are model parameters. The random variable $U$ is distributed uniformly on the unit sphere $\{u\in\mathbb{R}^r:\|u\|=1\}$ and $\xi\geq 0$ is a random variable in $\mathbb{R}$ independent of $U$. The vector $\theta$ and the matrix $\Sigma=AA^T$ are called the location and the scatter of the elliptical distribution.
\end{definition}

For any unit vector $u$, the distribution of $\xi u^TU$ does not depend on $u$ because of the symmetry of $U$. Define $H(\cdot)$ to be the distribution function of $\xi u^TU$. Since there is a one-to-one relation between $H(\cdot)$ and the distribution of $\xi$, the distribution of $X=\theta+\xi AU$ is fully determined by the triplet $(\theta,\Sigma,H)$, and therefore we write the distribution as $E(\theta,\Sigma,H)$.

Note that $\Sigma$ and $H$ are not identifiable, this is because $\xi AU= (a\xi)(a^{-1}A)U$ for any $a>0$. To overcome this issue, we restrict $H$ to the following class
$$\mathcal{H}=\left\{H\text{ is a distribution function}: H(t)+H(-t)\equiv1, \int R(|t|)dH(t)=\int R(|t|)d\Phi(t)\right\},$$
where $\Phi(\cdot)$ is the distribution function of $N(0,1)$, and
\begin{equation}
R(|t|)=\begin{cases}
|t|, & |t|\leq 1, \\
1, & |t|>1,
\end{cases} \label{eq:def-S}
\end{equation}
which is recognized as the clipped $\ell_1$ function. The restriction $H\in\mathcal{H}$ is without loss of generality. This is because the function $F(a)=\mathbb{E}R(|a\xi u^TU|)$ is increasing for all $a>0$, so that the equation $F(a)=\int R(|t|)d\Phi(t)$ must have a solution. Here, we do not use the simpler absolute function, because the first moment of $\xi u^TU$ may not exist.
\begin{definition}
The elliptical distribution $X=\theta+\xi AU$ has a canonical parametrization $(\theta,\Sigma,H)$ with $\Sigma=AA^T$ and $H\in\mathcal{H}$. We use the notation $E(\theta,\Sigma,H)$ to denote the elliptical distribution in its canonical form.
\end{definition}
With the canonical representation, the parameters $\theta,\Sigma,H$ are all identifiable. The scatter matrix $\Sigma$ is proportion to the covariance matrix whenever the covariance matrix exists. Moreover, for multivariate Gaussian $N(\theta,\Sigma)$, its canonical parametrization is $(\theta,\Sigma,\Phi)$, and the scatter matrix and the covariance matrix are identical.

The goal of this section is to estimate both the location $\theta$ and the scatter $\Sigma$ with observations
\begin{equation}
X_1,...,X_n \stackrel{iid}{\sim} P\quad\text{for some }P\text{ satisfying }\TV(P,E(\theta,\Sigma,H))\leq\epsilon. \label{eq:strong-con-ellip}
\end{equation}
To achieve this goal, we further require that $H$ belongs to the following class
$$\mathcal{H}(M')=\left\{H\in\mathcal{H}: \int_{1/4}^{1/3}dH(t)\geq \frac{1}{M'}\right\}.$$
The number $M'>0$ is assumed to be some large constant. It is easy to see that $\mathcal{H}=\cup_{M'>0}\mathcal{H}(M')$. The regularity condition $H\in\mathcal{H}(M')$ will be easily satisfied as long as there is a constant probability mass of $H$ contained in the interval $[1/4,1/3]$. This condition prevents some of the probability mass from escaping to infinity.

Define the estimator
\begin{equation}
(\wh{\theta},\wh{\Sigma},\wh{H}) = \argmin_{\eta\in\mathbb{R}^p,\Gamma\in\mathcal{E}_p(M),H\in\mathcal{H}(M')}\max_{T\in\mathcal{T}}\left[\frac{1}{n}\sum_{i=1}^n S(T(X_i),1) + \mathbb{E}_{X\sim E(\eta,\Gamma,G)}S(T(X),0)\right]. \label{eq:ellip-GAN}
\end{equation}
To accommodate for the more general generator class in (\ref{eq:ellip-GAN}), we consider the  discriminator class $\bar{\mathcal{T}}^L(\kappa,B)$, which has the same definition (\ref{eq:net-structure-deep}), except that
$$\mathcal{G}^1(B)=\mathcal{G}_{\ramp}=\left\{g(x)=\ramp(u^Tx+b):u\in\mathbb{R}^p,b\in\mathbb{R}\right\}.$$
In other words, $\bar{\mathcal{T}}^L(\kappa,B)$ and ${\mathcal{T}}^L(\kappa,B)$ only differs in the choice of the nonlinear activation function of the bottom layer. We remark that the discriminator class ${\mathcal{T}}^L(\kappa,B)$ also works for the elliptical distributions, but the theory would require a condition that is less transparent. The theoretical guarantee of the estimator (\ref{eq:ellip-GAN}) is given by the following theorem.
\begin{thm}\label{thm:ellip}
Consider the estimator (\ref{eq:ellip-GAN}) that is induced by a regular proper scoring rule that satisfies Condition \ref{cond:G}. The discriminator class is specified by $\mathcal{T}=\bar{\mathcal{T}}^L(\kappa,B)$ with the dimension of $(w_j)$ to be at least $2$. Assume $\frac{p}{n}+\epsilon^2\leq c$ for some sufficiently small constant $c>0$. Set $2\leq L=O(1)$, $1\leq B=O(1)$, and $\kappa=O\left(\sqrt{\frac{p}{n}}+\epsilon\right)$. Then, under the data generating process (\ref{eq:strong-con-ellip}), we have
\begin{eqnarray*}
\|\wh{\theta}-\theta\|^2 &\leq& C\left(\frac{p}{n}\vee\epsilon^2\right), \\
\opnorm{\wh{\Sigma}-\Sigma}^2 &\leq& C\left(\frac{p}{n}\vee\epsilon^2\right),
\end{eqnarray*}
with probability at least $1-e^{-C'(p+n\epsilon^2)}$ uniformly over all $\theta\in\mathbb{R}^p$, all $\opnorm{\Sigma}\leq M=O(1)$, and all $H\in\mathcal{H}(M')$ with $M'=O(1)$. The constants $C,C'>0$ are universal.
\end{thm}


\section{Numerical Studies}\label{sec:num}

We present our numerical results in this Section. We first give details of implementations in Section \ref{sec:implement}. We then compare our proposed methods with other methods in the literature in Section \ref{subsect:comp}. Simultaneous location and scatter estimation and comparisons of different scoring rules are investigated in Section \ref{subsect:simest} and Section \ref{subsect:sr}, respectively.

\subsection{Implementations}\label{sec:implement}

In order to implement the proposed methods for scatter estimation, we need to specify the generator networks. Depending on whether the data is centered and whether the distribution family is known, we consider the following four types of generator networks:
\begin{itemize}[noitemsep,topsep=0pt]
\item $G_1(Z; A) = AZ$. The random vector $Z$ is sampled from $E(0,I_p,H)$. Then, according to Definition \ref{def:elliptical}, we have $G_1(Z; A)\sim E(0,\Sigma,H)$ with $\Sigma=AA^T$. This is the simplest generator network suitable for centered observations with a known distribution family.
\item $G_2(U, z; A, w_g) = g_{w_g}(z)AU$. When the distribution family is unknown, we can represent an $E(0,AA^T,H)$ random variable by $\xi AU$ according to Definition \ref{def:elliptical}. This leads to the generator network $G_2(U, z; A, w_g)$, where we model $\xi$ by $g_{w_g}(z)$, a neural work  with parameter $w_g$ and input vector $z\sim N(0,I_q)$. The input dimension $q$ will be specified later.
\item $G_3(Z; A, \theta) = \theta + AZ$. This is an extension of $G_1$ when the observations are not centered.
\item $G_4(U, z; A, w_g, \theta) = \theta + g_{w_g}(z)AU$. This is an extension of $G_2$ when the observations are not centered.
\end{itemize}
The algorithm to implement JS-GAN \citep{goodfellow2014generative} with the generator network $G_1(Z; A)$ is given below.
\begin{algorithm}[H]
\caption{$\argmin_{A}\max_{w}\left[\frac{1}{n}\sum_{i=1}^n \log D_{w}(X_i) + \Expect_{Z\sim E(0,I_p,H)} \log(1-D_{w}(G_1(Z; A)))\right]$}\label{alg:js}
\textbf{Input}: 

1. Observations $S=\{X_{1},\ldots,X_{n}\}\subset\mathbb{R}^{p}$, \\
2. Learning rates $\gamma_{d}$/$\gamma_{g}$ for the discriminator/generator, \\
3. Batch size $m$, \\
4. Iterations for discriminator/generator steps in each epoch $K_d$/$K_g$, \\
5. Total epochs $T$, \\
6. Average epochs $T_0$. \\
\textbf{Initialization}:

1. Initialize $\wh{\Sigma}_0$ by scaled Kendall's $\tau$, apply singular value decomposition $\wh{\Sigma}_0=V_0\Lambda_0V_0^T$, and set $A_0=V_0\Lambda_0^{1/2}$. \\
2. Initialize the discriminator network by Xavier \citep{glorot2010understanding}, where the first layer of $D_{w}(x)$ is initialized by some Gaussian distribution.
\begin{algorithmic}[1]
	\For{\texttt{$t=1,\ldots,T$}}
		\For{\texttt{$k=1,\ldots,K_d$}}
			\State Sample mini-batch $\{X_{1},\ldots,X_{m}\}$ from $S$. Sample $\{Z_{1},\ldots,Z_{m}\}$ from $E(0,I_p,H)$;
			\State $g_{w}\gets\nabla_{w}\left[\frac{1}{m}\sum_{i=1}^{m} \log D_{w}(X_{i}) + \frac{1}{m}\sum_{i=1}^{m} \log (1 - D_{w}(G_1(Z_{i}; A)))\right]$;
			\State $w \gets w + \gamma_{d}g_{w}$;
		\EndFor
		\For{\texttt{$k=1,\ldots,K_g$}}
			\State Sample $\{Z_{1},\ldots,Z_{m}\}$ from $E(0,I_p,H)$;
			\State $g_{A}\gets\nabla_{A}\left[\frac{1}{m}\sum_{i=1}^{m} \log(1-D_{w}(G_1(Z_{i}; A)))\right]$;
			\State $A \gets A - \gamma_{g}g_{A}$;
		\EndFor
	\EndFor
\end{algorithmic}
\textbf{Return}: The average over the last $T_0$ epochs $\wh{\Sigma} = \frac{1}{T_0}\sum_{t=T-T_0+1}^{T}A_tA_t^T$.
\end{algorithm}
Several remarks for Algorithm \ref{alg:js} are given below:
\begin{itemize}[noitemsep,topsep=0pt]
\item {\it Some variations and extra details.} For simplicity, we only state Algorithm \ref{alg:js} that covers the case of JS-GAN with the generator network $G_1(Z; A)$. Modifications to accommodate other proper scoring rules and more complicated generator networks are straightforward, and thus the details are omitted. An extra implementation detail is that after every $T_1$ epochs, we update the learning rates $\gamma_g$ and $\gamma_d$ by $\alpha\gamma_g$ and $\alpha\gamma_d$ with some $\alpha\in(0,1)$ for better convergence.
\item {\it Discriminator network structures.} The structure of the discriminator network is set as $p$-$2p$-$\floor{p/2}$-$1$. We consistently observe that this wide structure outperforms a narrow ones such as $p$-$\floor{p/2}$-$\floor{p/4}$-$1$. The choice of nonlinearity follows that of (\ref{eq:net-structure-4}), except we use ${\sf LeakyReLU}(x)=\max(0.2x,x)$ instead of $\relu(x)$ to avoid vanishing gradients. This slight change will not affect the theoretical results in the paper.
\item {\it Identifiability and calibration.} Suppose $X\sim E(\theta,AA^T,H)$ so that it has the representation $X=\theta + \xi AU$. Then, the matrix $\wh{A}\wh{A}^T$ output by Algorithm \ref{alg:js} is an estimator for $AA^T$ up to a multiplicative constant due to the identifiability issue discussed in Section \ref{sect:elliptical}. Therefore, we need to define the version of $AA^T$ and calibrate the estimator $\wh{A}\wh{A}^T$ accordingly. In Section \ref{sect:elliptical}, the multiplicative constant is determined through the equation $\mathbb{E}R(|a\xi u^TU|)=\int R(|t|)d\Phi(t)$. Note that here the choice of $\Phi$ is arbitrary, and thus in our numerical studies, it is more convenient to replace the Gaussian cumulative distribution function (CDF) $\Phi$ by the CDF of the distribution that we are working with. For example, consider a multivariate $t$-distribution $T_v(\theta,\Sigma)$ with density proportional to $\left(1+(x-\theta)^T\Sigma^{-1}(x-\theta)/v\right)^{-(v+p)/2}$, we will find a factor $a>0$ such that $\mathbb{E}_{U,z}R(ag_{\wh{w}_g}(z)u^TU)=\mathbb{E}_{X\sim T_v(0,1)}R(|X|)$ is satisfied, where $U$ is uniformly distributed on the unit sphere and $z\sim N(0,I_q)$. Our final estimator is given by $a^{-2}\wh{A}\wh{A}^T$. With this scaling, the estimator is directly targeted at the $\Sigma$ in the formula of the density of $T_v(\theta,\Sigma)$.
\end{itemize}
The following table summarizes the hyperparameters that can reproduce our numerical results for $p=100$. Hyperparameters for other dimensions can be found in \url{https://github.com/zhuwzh/Robust-GAN-Scatter}.
\begin{table}[H]
\centering
\resizebox{\linewidth}{!}{
  \begin{tabular}{|c|c|c|c|c|c|c|c|}
    \hline
    generator & $\xi$ network $g_{w_g}(z)$ & discriminator & $\sigma_1$ & $\gamma_d / \gamma_g$ & $K_d/K_g$ & $T/T_0$ & $\alpha / T_1$ \\	
    \hline\hline
    $G_1/G_3$ & - & 100-200-50-1 & 0.0025 & 0.025/0.1 & 12/3 & 500/25 & 0.2/200 \\
    \hline
    $G_2/G_4$ & 48-32-24-12-1 & 100-200-50-1 & 0.025 & 0.05/0.025 & 12/3 & 500/25 & 0.2/200\\
    \hline
  \end{tabular}
}
\caption{Hyperparameters used in our numerical studies. Here, $\sigma_1$ is the standard deviation of the Gaussian initialization in the first layer (next to the input) of the discriminator network. Other parameters are introduced in Algorithm \ref{alg:js}.}\label{tab:hyper}
\end{table}

\subsection{Comparisons with Other Methods}\label{subsect:comp}

In this section, we compare the performance of JS-GAN against other methods for robust scatter estimation in the literature.
We first introduce some other robust scatter matrix estimators that we will compare with. The definitions of these robust matrix estimator are all up to some scaling factor. Tyler's M-estimator \citep{tyler1987distribution} is defined as a solution of
$\sum_{i=1}^n\frac{X_iX_i^T}{X_i^T\Sigma^{-1}X_i}=c\Sigma$ for some $c>0$.
Note that it is a special case of Maronna's M-estimator \citep{maronna1976robust}. Properties of Tyler's M-estimator were studied by \cite{dumbgen1998tyler,dumbgen2005breakdown,zhang2002some,zhang2016marvcenko}. The second robust estimator of scatter that we will compare with is the scaled Kendall's $\tau$. The Kendall's $\tau$ correlation coefficient \citep{kendall1938new} between the $j$th and the $k$th variables is defined as
$$\hat{\tau}_{jk}=\frac{2}{n(n-1)}\sum_{i<i'}\text{sign}\left((X_i-X_{i'})_j(X_i-X_{i'})_k\right).$$
Then, $\wh{K}=(\wh{K}_{jk})$ with $\wh{K}_{jk}=\sin\left(\frac{\pi}{2}\hat{\tau}_{jk}\right)$ is an estimator of the correlation matrix \citep{kruskal1958ordinal,han2014scale}. To obtain an estimator for the scatter matrix, define a diagonal matrix $\wh{S}$ with diagonal entries $\wh{S}_{jj}=\textsf{Median}(\{X_{ij}^2\}_{i=1}^n)$. Then, the scaled Kendall's $\tau$ estimator for the scatter matrix is $\wh{S}^{1/2}\wh{K}\wh{S}^{1/2}$.
Thirdly, we introduce the minimum volume ellipsoid estimator (MVE) by \cite{leroy1987robust}. It finds the ellipsoid covering at least $n/2$ points of $\{X_i\}_{i=1}^n$ with the minimum volume and then use the shape of the ellipsoid as the scatter matrix estimator. Properties of MVE have been studied by \cite{davies1992asymptotics}. Finally, we consider the dimension halving method proposed by \cite{lai2016agnostic} based on the idea of higher moment certification. We remark that among all the methods that we compare here, dimension halving is the only method that is designed for Huber's contamination model.
For comparison of performances, we rescale all the estimators by some constant factors so that all of them are targeted at the same population scatter matrix.\footnote{Dimension halving is designed to estimate the covariance matrix when it exists. For the t-distribution with degrees of freedom $v$, the final estimator needs to be scaled by $\frac{v}{v-2}$ when $v>2$. The results for $v\in\{1,2\}$ are omitted for dimension halving because the covariance does not exist.}

The comparisons cover the following scenarios from different perspectives.

\paragraph{Influence of Tail.}

We first study the influence of the tail behavior. We consider i.i.d. observations from $(1-\epsilon)T_v(0,\Sigma_{\sf{ar}})+\epsilon T_v(5 \mathds{1}_p, 5I_p)$, where $(\Sigma_{\sf{ar}})_{jk} = (1/2)^{|j-k|}$ and $\mathds{1}_p$ stands for the $p$-dimensional vector with all entries $1$'s. Note that the second moment of the multivariate $t$-distribution exists only when $v>2$.

\begin{table}[H]
   \centering
   \resizebox{\linewidth}{!}{%
   \begin{tabular}{|c|c|c|c|c|c|c|}
    \hline
     degrees of freedom $v$ & $G_1(Z; A) = AZ$ & $G_2(U, z; A, w_g) = g_{w_g}(z)AU$ & Dimension Halving & Tyler's M-estimator & Kendall's $\tau$ & MVE \\
     \hline\hline
     $1$ & 0.2808 (0.0440) & 0.3350 (0.0681) & - & 372.9637 (582.3385) & 52.5653 (0.6361) & 50.2995 (0.6259) \\
     \hline
     $2$ & 0.3450 (0.0157) & 0.4059 (0.0254) & - & 55.5152 (1.1901) & 64.7625 (0.4798) & 20.1941 (1.8645) \\
     \hline
     $4$ & 0.2751 (0.0147) & 0.2775 (0.0456) & 1.2834 (0.0512) &  38.7569 (0.2740) & 72.8037 (0.3369) & 0.1920 (0.0299) \\
     \hline
     $8$ & 0.2131 (0.0162) & 0.2113 (0.0306) & 0.8902 (0.0728) & 39.0265 (0.2014) & 77.2117 (0.3486) &  0.1753 (0.0218) \\
     \hline
     $16$ & 0.1764 (0.0120) & 0.2076 (0.0210)  & 0.8354 (0.0926) & 39.1167 (0.3200) & 79.2252 (0.2728) & 0.1683 (0.0136) \\
     \hline
     $32$ & 0.1576 (0.0067) & 0.2056 (0.0202) & 0.8572 (0.0687) & 39.1985 (0.2153) & 80.2075 (0.1706) & 0.1493 (0.0085) \\
      \hline
   \end{tabular}}
\caption{Simulation results with $n=50,000, p=100, \epsilon=0.2$ and $v\in\{1,2,4,8,16,32\}$. We show the average error $\opnorm{\wh{\Sigma} - \Sigma}$ in each cell with standard deviation in parenthesis from 10 repeated experiments.}\label{tab:tail}
\end{table}

Table \ref{tab:tail} summarizes the results with the degrees of freedom $v$ varying from $1$ to $32$. We observe that JS-GANs (with generator networks $G_1$ and $G_2$) are overall the two best methods, especially when $v\in\{1,2\}$. Dimension halving performs better than Tyler's M-estimator and Kendall's $\tau$, because it is the only one that is designed for Huber's contamination model among all other methods. Among the remaining three methods, MVE is greatly influenced by the value of $v$, while Kendall's $\tau$ and Tyler's M-estimator does not seem to be robust in this setting.

\paragraph{Distance of Contamination.}

We then study the effect of the distance of the contamination distribution. In this experiment, we sample i.i.d. observations from $(1-\epsilon)N(0,I_p)+\epsilon \delta_{s\mathds{1}_p}$. That is, we model the contamination by a Dirac distribution with each coordinate being $s$, and we vary the value of $s$ in this experiment.

\begin{table}[H]
   \centering
   \resizebox{\linewidth}{!}{%
   \begin{tabular}{|c|c|c|c|c|c|c|}
    \hline
     factor $s$ & $G_1(Z; A) = AZ$ & $G_2(U, z; A, w_g) = g_{w_g}(z)AU$ & Dimension Halving & Tyler's M-estimator & Kendall's $\tau$ & MVE\footnotemark \\
     \hline\hline
     $0.5$ & 0.2000 (0.0251) & 0.2057 (0.0104)  & 0.1949 (0.0085) & 98.7102 (3.7054) & 3.8098 (0.0158) & 5.1546 (0.0424) \\
     \hline
     $1$ & 0.1699 (0.0168) & 0.1607 (0.0084) & 0.2163 (0.0077) & 103.3438 (2.4459) & 23.4627 (0.0732)  & 17.5985 (0.0964) \\
     \hline
     $2$ & 0.1705 (0.0145) & 0.1576 (0.0097) & 0.2762 (0.0159) & 137.6124 (1.1288)   & 84.5634 (0.0750) & 45.9484 (0.1704) \\
     \hline
     $4$ & 0.1530 (0.0059) & 0.1557 (0.0142) & 0.2982 (0.0110) & 264.7385 (0.5537) & 92.0462 (0.1206) & 76.8558 (0.2466) \\
     \hline
     $8$ & 0.1600 (0.0094) & 0.1555 (0.0116) & 0.3008 (0.0184) & 366.0933 (0.3620) & 92.1079 (0.1124) & 92.4466 (0.0811) \\
     \hline
   \end{tabular}}
\caption{Simulation results with $n=50,000, p=100, \epsilon=0.2$ and $s\in\{0.5,1,2,4,8\}$. We show the average error $\opnorm{\wh{\Sigma} - \Sigma}$ in each cell with standard deviation in parenthesis from 10 repeated experiments.}\label{tab:similarity}
\end{table}
\footnotetext{We apply the function {\rm cov.rob} in the R package MASS. However, the function cannot be applied when some observations are linearly dependent. We add a small random perturbation $N(0, 0.01)$ to the original data before applying {\rm cov.rob}.}

The results summarized in Table \ref{tab:similarity} show that JS-GANs and dimension halving are much better than the other three methods. The errors of Kendall's $\tau$ and MVE exhibit clear increasing patterns as $s$ grows. Tyler's M-estimator does not seem to be robust under this setting.

It is also interesting to note that JS-GANs and dimension halving have different error behaviors. The error of dimension halving grows slowly as $s$ increases, while the errors of JS-GANs decrease. This is because as the contamination distribution $\delta_{s\mathds{1}_p}$ is moving away from $N(0,I_p)$, the discriminator network in GANs gets better at distinguishing contamination from the true model.

\paragraph{Dependence on $(\epsilon,n,p)$.}

Lastly, we show the comparison results in Tables \ref{tab:eps}-\ref{tab:dim} when the contamination proportion $\epsilon$, the sample size $n$, and the dimension $p$ vary. We consider i.i.d. observations from $(1-\epsilon)N(0,\Sigma_{\sf{ar}})+\epsilon Q$ with $(\Sigma_{\sf{ar}})_{jk} = (1/2)^{|j-k|}$ and $Q\in\{N(5 \mathds{1}_p, 5I_p),\delta_{5\mathds{1}_p}\}$.

\begin{table}[H]
   \centering
   \resizebox{\linewidth}{!}{%
   \begin{tabular}{|c|c|c|c|c|c|c|c|}
    \hline
     $\epsilon$ & $Q$ & $G_1(Z; A) = AZ$ & $G_2(U, z; A, w_g) = g_{w_g}(z)AU$ & Dimension Halving & Tyler's M-estimator & Kendall's $\tau$ & MVE \\
     \hline\hline
     \multirow{2}{*}{0.02} & $N(5 \mathds{1}_p, 5I_p)$ & 0.1438 (0.0109)& 0.1829 (0.0221)  & 0.7709 (0.0720) & 4.1114 (0.0410) & 5.9289 (0.0362) & 0.1286 (0.0115) \\
     & $\delta_{5\mathds{1}_p}$ & 0.1373 (0.0053) & 0.1782 (0.0106) &  0.7758 (0.0688) & 105.4027 (2.1131) & 6.4535 (0.0385) & 20.0688 (17.1654)\\
     \hline
     \multirow{2}{*}{$0.05$} & $N(5 \mathds{1}_p, 5I_p)$ & 0.1434 (0.0096) & 0.1834 (0.0123) & 0.7633 (0.0639) & 11.4392 (0.0739)  & 15.3559 (0.0430) & 0.1313 (0.0117) \\
     & $\delta_{5\mathds{1}_p}$ & 0.1443 (0.0076) & 0.1888 (0.0172) & 0.7781 (0.0533) & 130.9782 (0.3241) & 16.8037 (0.0559) & 49.7855 (17.4420) \\
     \hline
     \multirow{2}{*}{$0.1$} & $N(5 \mathds{1}_p, 5I_p)$ & 0.1500 (0.0111) & 0.1881 (0.0143) & 0.7548 (0.0435) & 22.5296 (0.1511) & 33.1473 (0.0867) & 0.1420 (0.0169)\\
     & $\delta_{5\mathds{1}_p}$ & 0.1470 (0.0072)& 0.1957 (0.0135) & 0.7740 (0.0678) & 198.0822 (0.6326) & 36.5264 (0.0924) & 63.8083 (22.3797) \\
     \hline
     \multirow{2}{*}{$0.2$} & $N(5 \mathds{1}_p, 5I_p)$ & 0.1620 (0.0102) & 0.2000 (0.0157) & 0.7285 (0.0862) & 39.3082 (0.2273) & 81.2740 (0.1081) & 0.1420 (0.0107 \\
     & $\delta_{5\mathds{1}_p}$ & 0.1632 (0.0206) & 0.1938 (0.0155) & 0.9365 (0.0788) & 235.5581 (0.8650) & 92.1920 (0.1281) & 81.8603 (0.1722) \\
     \hline
   \end{tabular}}
\caption{Simulation results with $n=50,000, p=100$ and $\epsilon\in\{0.02,0.05,0.1,0.2\}$.
We show the average error $\opnorm{\wh{\Sigma} - \Sigma}$ in each cell with standard deviation in parenthesis from 10 repeated experiments. }\label{tab:eps}
\end{table}

\begin{table}[H]
   \centering
   \resizebox{\linewidth}{!}{%
   \begin{tabular}{|c|c|c|c|c|c|c|c|}
    \hline
     $n$ & $Q$ & $G_1(Z; A) = AZ$ & $G_2(U, z; A, w_g) = g_{w_g}(z)AU$ & Dimension Halving & Tyler's M-estimator & Kendall's $\tau$ & MVE \\
     \hline\hline
     \multirow{2}{*}{$5,000$} & $N(5 \mathds{1}_p, 5I_p)$ & 1.0186 (0.0677) & 0.6449 (0.0576) & 1.0813 (0.0593) & 38.6494 (0.4477) & 81.0570 (0.4043) & 0.4724 (0.0623) \\
     & $\delta_{5\mathds{1}_p}$ & 0.9142 (0.0717) &  0.6152 (0.0751) & 1.1796 (0.1257) & 233.2116 (2.9582) & 92.0669 (0.2453) & 82.1385 (0.3717) \\
     \hline
     \multirow{2}{*}{$10,000$}  & $N(5 \mathds{1}_p, 5I_p)$ & 0.3465 (0.0722) & 0.5852 (0.0816) & 0.9489 (0.0563) & 39.1785 (0.5876) & 81.1670 (0.2038) & 0.3225 (0.0350) \\
     & $\delta_{5\mathds{1}_p}$ & 0.3293 (0.0166) & 0.4244 (0.0632) & 1.0558 (0.0796) & 235.6172 (1.8062) & 92.1067 (0.3140) & 82.1225 (0.4402) \\
     \hline
     \multirow{2}{*}{$20,000$}  & $N(5 \mathds{1}_p, 5I_p)$ & 0.2382 (0.0084)& 0.2867 (0.0183) & 0.7934 (0.0527) & 39.1790 (0.2643)  & 81.1848 (0.1246) & 0.2198 (0.0162)  \\
     & $\delta_{5\mathds{1}_p}$ & 0.2459 (0.0151) & 0.2835 (0.0287) & 0.9463 (0.0793) & 234.2984 (1.8534) & 92.1426 (0.2296) & 81.9923 (0.3232) \\
     \hline
     \multirow{2}{*}{$50,000$}  & $N(5 \mathds{1}_p, 5I_p)$ & 0.1547 (0.0082) & 0.1866 (0.0211) & 0.8185 (0.0847) & 39.2105 (0.2365) & 81.2640 (0.0727) & 0.1447 (0.0127) \\
     & $\delta_{5\mathds{1}_p}$ & 0.1580 (0.0101) & 0.1947 (0.0201) & 0.8988 (0.0821) & 235.0871 (0.9106) & 92.1250 (0.1733) & 82.0914 (0.2233) \\
     \hline
   \end{tabular}}
\caption{Simulation results with $p=100, \epsilon=0.2$ and $n\in\{5000,10000,20000,50000\}$.
We show the average error $\opnorm{\wh{\Sigma} - \Sigma}$ in each cell with standard deviation in parenthesis from 10 repeated experiments.}\label{tab:size}
\end{table}

\begin{table}[H]
   \centering
   \resizebox{\linewidth}{!}{%
   \begin{tabular}{|c|c|c|c|c|c|c|c|}
    \hline
     $p$ & $Q$ & $G_1(Z; A) = AZ$ & $G_2(U, z; A, w_g) = g_{w_g}(z)AU$ & Dimension Halving & Tyler's M-estimator & Kendall's $\tau$ & MVE \\
     \hline\hline
     \multirow{2}{*}{10} & $N(5 \mathds{1}_p, 5I_p)$ & 0.0498 (0.0116) & 0.0541 (0.0139) & 0.5800 (0.0268) & 2.8370 (0.0345) & 8.1053 (0.0540) & 0.0475 (0.0133) \\
     & $\delta_{5\mathds{1}_p}$ & 0.0519 (0.0093) & 0.1072 (0.0424) & 0.5587 (0.0236) & 10.8566 (0.0327) & 9.2230 (0.0838) & 6.3590 (0.0458) \\
     \hline
     \multirow{2}{*}{25} & $N(5 \mathds{1}_p, 5I_p)$ & 0.0813 (0.0127) &0.1059 (0.0106) & 0.6090 (0.0256) & 8.7515 (0.0373) & 20.2953 (0.0788) & 0.0637 (0.0084) \\
     & $\delta_{5\mathds{1}_p}$ & 0.0813 (0.0089) & 0.1188 (0.0269) & 0.6010 (0.0261) & 43.3639 (0.2047) & 23.0388 (0.1104) & 18.8767 (0.0685) \\
     \hline
     \multirow{2}{*}{50} & $N(5 \mathds{1}_p, 5I_p)$ & 0.1171 (0.0111) & 0.1296 (0.0089) & 0.6238 (0.0197) & 18.9115 (0.1781) & 40.6553 (0.0731) & 0.1034 (0.0081)  \\
     & $\delta_{5\mathds{1}_p}$ & 0.1044 (0.0090) & 0.1444 (0.0174) & 0.6420 (0.0438) & 125.5211 (0.5539) & 46.0510 (0.0870) & 40.0609 (0.2258) \\
     \hline
     \multirow{2}{*}{100} & $N(5 \mathds{1}_p, 5I_p)$ & 0.1538 (0.0052) & 0.1893 (0.0131) & 0.6406 (0.0211)  & 39.2077 (0.1659) & 81.1823 (0.0865) & 0.1454 (0.0147) \\
     & $\delta_{5\mathds{1}_p}$ & 0.1552 (0.0087) & 0.1957 (0.0150) & 0.6385 (0.0279) & 235.5848 (1.2379) & 92.0734 (0.0988) & 82.0031 (0.2379) \\
     \hline
     \multirow{2}{*}{200} & $N(5 \mathds{1}_p, 5I_p)$ & 0.2154 (0.0066) & 0.2630 (0.0125) & 0.6824 (0.0274) & 80.0833 (0.5008) & 162.3562 (0.1491) & 0.2115 (0.0100) \\
     & $\delta_{5\mathds{1}_p}$ & 0.2159 (0.0057) & 0.2559 (0.0099) & 0.6628 (0.0206) & 731.8936 (1.5718) & 184.2577 (0.1533) & 166.0049 (0.3709) \\
     \hline
   \end{tabular}}
\caption{Simulation results with $n=50,000, \epsilon=0.2$ and $p\in\{10,25,50,100,200\}$.
We show the average error $\opnorm{\wh{\Sigma} - \Sigma}$ in each cell with standard deviation in parenthesis from 10 repeated experiments.}\label{tab:dim}
\end{table}

The results are summarized in Tables \ref{tab:eps}-\ref{tab:dim}. As $\epsilon$ gets larger, Table \ref{tab:eps} shows that the errors of Tyler's M-estimator, Kendall's $\tau$ and MVE all grow significantly, while the dependence on $\epsilon$ for the other three estimators are very mild. The same pattern also appears in Table \ref{tab:dim} as the dimension $p$ grows. For Table \ref{tab:size}, we observe that the errors of the two GANs and dimension halving all decrease as $n$ grows. In contrast, the errors of Tyler's M-estimator and Kendall's $\tau$ almost stay as constant when $n$ varies because of the contamination. MVE exhibits different behaviors for the two contamination distributions. Its error decreases as $n$ grows when the contamination distribution is $N(5 \mathds{1}_p, 5I_p)$ and stays as constant when the contamination distribution is $\delta_{5\mathds{1}_p}$.

In summary, our numerical results show that Tyler's M-estimator and Kendall's $\tau$ do not work well under Huber's contamination model. MVE works well for certain contamination distributions but is certainly not robust against all contamination distributions. The GANs proposed in the paper and dimension halving all work very well because they are designed for robust estimation under Huber's contamination model. Among the three, the GANs constantly have smaller errors than dimension halving. This is not surprising given the minimax optimality of the GANs under Huber's contamination model established in this paper.

\subsection{Simultaneous Estimation of Location and Scatter}\label{subsect:simest}

In this section, we study robust simultaneous location and scatter estimation. Our goal is to compare the performances of the four generator networks $G_1$, $G_2$, $G_3$ and $G_4$. Table \ref{tab:simest} summarizes the numerical results under four different settings of contamination models.
We observe that the networks $G_3$ and $G_4$ have similar performances in terms of estimating the scatter as the networks $G_1$ and $G_2$ that only estimate the scatter. We can also compare $G_2$ with $G_1$ and $G_4$ with $G_3$, since $G_1$ and $G_3$ assume the knowledge of the distribution family, while $G_2$ and $G_4$ estimate the distribution via the additional $g_{w_g}(z)$. Even though the additional knowledge of the distribution family does help $G_1$ and $G_3$ to outperform $G_2$ and $G_4$, the advantage is not significant in terms of estimating the scatter and is negligible in terms of location estimation. These observations imply that the most complicated network $G_4$ works very well for adaptive estimation under the general elliptical distribution family.
\begin{table}[H]
\begin{center}
   \resizebox{\linewidth}{!}{
   \begin{tabular}{|c|c|c|c|c|c|c|}
   \hline
   	  \multirow{2}{*}{$(P, Q)$} & $G_1(z;A)=Az$ & \multicolumn{2}{c|}{$G_3(z;A,\mu)=Az+\mu$} &  $G_2(u,z;A,w_g)= g_{w_g}(z)Au$ & \multicolumn{2}{c|}{$G_4(u,z;A,w_g,\mu)=g_{w_g}(z)Au+\mu$} \\\cline{2-7}
    	& $\opnorm{\wh{\Sigma}-\Sigma}$ & $\opnorm{\wh{\Sigma}-\Sigma}$ & $\|\wh{\theta}-\theta\|$ & $\opnorm{\wh{\Sigma}-\Sigma}$ & $\opnorm{\wh{\Sigma}-\Sigma}$ & $\|\wh{\theta}-\theta\|$  \\ 
	\hline\hline
	$(N(0, I_p), N(5,5I_p))$ & 0.1615 (0.0134) &  0.1537 (0.0155) &  0.0508 (0.0054) & 0.1624 (0.0141) & 0.1694 (0.0105) & 0.0519 (0.0048) \\
	\hline
	$(N(0, \Sigma_{\sf{ar}}), \delta_{4\mathds{1}_p})$ & 0.1530 (0.0059) & 0.1640 (0.0106) & 0.0547 (0.0039) & 0.1557 (0.0142) &  0.1880 (0.0134) & 0.0544 (0.0073) \\
	\hline
	$(T_1(0, \Sigma_{\sf{ar}}), T_1(5, 5I_p))$ & 0.2808 (0.0440) & 0.2512 (0.0479) & 0.0656 (0.0065) & 0.3350 (0.0681) & 0.4678 (0.0498) & 0.0575 (0.0048) \\
	\hline
	$(T_2(0, \Sigma_{\sf{ar}}), T_2(5, 5I_p))$ & 0.3450 (0.0157) & 0.3743 (0.0097) & 0.0640 (0.0056) & 0.4059 (0.0254) & 0.4704 (0.0299) & 0.0642 (0.0040) \\
	\hline
   \end{tabular}}
\caption{Simulation results with i.i.d. observations generated from $(1-\epsilon)P+\epsilon Q$, where $n=50,000, p=100$ and $\epsilon=0.2$. We show the average errors $\opnorm{\wh{\Sigma} - \Sigma}$ and $\|\wh{\theta}-\theta\|$ in each cell with standard deviation in parenthesis from 10 repeated experiments.}\label{tab:simest}
\end{center}
\end{table}

\subsection{Comparisons of Proper Scoring Rules}\label{subsect:sr}

As we have shown in Section \ref{subsect:sc-example}, JS-GAN can be understood as a special case of a more general class Beta-GAN$(\alpha, \beta)$ with $\alpha,\beta>-1$. Our theoretical results are valid for any Beta-GAN$(\alpha, \beta)$ as long as $|\alpha-\beta|<1$ (Condition \ref{cond:G}). In this section, we study the performance of this wide class of GANs with various choices of $(\alpha,\beta)$. Table \ref{tab:sc} summarizes the numerical results under four different settings of contamination models. The comparison includes JS-GAN and LS-GAN, which correspond to $(\alpha,\beta)=(0,0)$ and $(\alpha,\beta)=(1,1)$. In addition, we also consider $(\alpha,\beta)\in\{(-0.5,-0.5),(0.5,0.5), (0.5,1), (1,0.5), (2,2), (4,4)\}$. We observe that the GANs with $0\leq \alpha,\beta\leq 1$ all have very similar performances. On the other hand, the errors increase as the values of $\alpha$ and $\beta$ grow, which is shown in the last two columns of Table \ref{tab:sc} for $(\alpha,\beta)\in\{(2,2), (4,4)\}$. In fact, our additional experiments show that the performance of Beta-GAN$(\alpha, \beta)$ is not acceptable anymore as soon as $\alpha,\beta\geq 8$. This may be caused by the bad landscape of the objective function for large $\alpha$ and $\beta$. In fact, as $\alpha=\beta\rightarrow\infty$, we recover TV-GAN, which is known to have a bad landscape for robust estimation \citep{gao2018robust}. The boosting score, which corresponds to $(\alpha,\beta)=(-0.5,-0.5)$, leads to worse errors than JS-GAN and LS-GAN.

We also note that some asymmetric GANs (e.g. $(\alpha,\beta)=(1,0.5)$) have better performance than JS-GAN and LS-GAN. It is interesting to further explore the properties of different scores from both theoretical and experimental perspectives in the future work.
\begin{table}[H]
   \centering
   \resizebox{\linewidth}{!}{%
   \begin{tabular}{|c|c|c|c|c|c|c|c|c|}
    \hline
     $(P, Q, p)$ & Beta(-0.5,-0.5) & JS-GAN & Beta(0.5, 0.5) & Beta(0.5,1) & Beta(1,0.5) & LS-GAN & Beta(2,2) & Beta(4,4)  \\
     \hline\hline
     $(N(0, I_p), N(5, 5I_p), 100)$ & 0.1557 (0.0093) & 0.1188 (0.0046) & 0.1228 (0.0045) & 0.1201 (0.0033) & 0.1040 (0.0017) & 0.1283 (0.0095) & 0.1402 (0.0063) & 0.3478 (0.0035) \\
     \hline
     $(N(0, I_p), N(5, 5I_p), 200)$ & 0.3346 (0.0149) & 0.1720 (0.0032) & 0.1677 (0.0045) & 0.1697 (0.0054) & 0.1599 (0.0026) & 0.1749 (0.0048) & 0.1978 (0.0031) & 0.3508 (0.0034) \\
     \hline
     $(T_2(0, \Sigma_{\sf{ar}}), T_2(5, 5I_p), 100)$ & 0.5653 (0.5065) & 0.1848 (0.0106) & 0.1941 (0.0087) & 0.2016 (0.0190) & 0.1925 (0.0149) & 0.1882 (0.0152) &0.3371 (0.0378) & 0.9689 (0.1124) \\
     \hline
     $(T_4(0, \Sigma_{\sf{ar}}), T_4(5, 5I_p), 100)$ & 0.2726 (0.0083) & 0.2009 (0.0079) & 0.1923 (0.0125) & 0.2133 (0.0105) & 0.1758 (0.0122) & 0.1999 (0.0130) & 0.3334 (0.0213) &  0.7740 (0.0432) \\     
     \hline
   \end{tabular}}
\caption{Simulation results with i.i.d. observations generated from $(1-\epsilon)P+\epsilon Q$, where $n=50,000, p\in\{100,200\}$ and $\epsilon=0.2$. We show the average error $\opnorm{\wh{\Sigma} - \Sigma}$ in each cell with standard deviation in parenthesis from 10 repeated experiments.}\label{tab:sc}
\end{table}


\section{Proofs}\label{sec:proof}

\subsection{Some Lemmas}

Before proving the main results, we introduce some lemmas, whose proofs are given in Section \ref{sec:aux}.
\begin{lemma}\label{lem:complexity-T1}
Given i.i.d. observations $X_1,...,X_n\sim \mathbb{P}$, and $\mathcal{T}\in\{\mathcal{T}_1,\mathcal{T}_3\}$ defined by either (\ref{eq:net-structure-1}) or (\ref{eq:net-structure-3}), we have for any $\delta>0$,
$$\sup_{T\in\mathcal{T}}\left|\frac{1}{n}\sum_{i=1}^n\log T(X_i)-\mathbb{E}\log T(X)\right|\leq C\kappa\left(\sqrt{\frac{p}{n}}+\sqrt{\frac{\log(1/\delta)}{n}}\right),$$
with probability at least $1-\delta$ for some universal constant $C>0$.
\end{lemma}

\begin{lemma}\label{lem:complexity-T2}
Given i.i.d. observations $X_1,...,X_n\sim N(0,\Sigma)$ and the function class $\mathcal{T}_2$ defined by (\ref{eq:net-structure-2}). Assume $\opnorm{\Sigma}\leq M=O(1)$. We have for any $\delta>0$,
\begin{eqnarray}
\sup_{T\in\mathcal{T}_2}\left|\frac{1}{n}\sum_{i=1}^n\log T(X_i)-\mathbb{E}\log T(X)\right|\leq C\kappa\left(\sqrt{\frac{p}{n}}+\sqrt{\frac{\log(1/\delta)}{n}}\right),
\end{eqnarray}
with probability at least $1-\delta$ for some universal constant $C>0$.
\end{lemma}

\begin{lemma}\label{lem:complexity-T3}
Given i.i.d. observations $X_1,...,X_n\sim\mathbb{P}$ and the function class $\mathcal{T}_3$ defined by (\ref{eq:net-structure-3}). Assume $\{S(\cdot,1),S(\cdot,0)\}$ is a regular proper scoring rule that satisfies Condition \ref{cond:G} and $\kappa\leq c$ for some sufficiently small constant $c>0$. We have for any $\delta>0$,
$$\sup_{T\in\mathcal{T}_3}\left|\frac{1}{n}\sum_{i=1}^n S(T(X_i),1)-\mathbb{E}S(T(X),1)\right|\leq C\kappa\left(\sqrt{\frac{p}{n}}+\sqrt{\frac{\log(1/\delta)}{n}}\right),$$
with probability at least $1-\delta$ for some universal constant $C>0$.
\end{lemma}

\begin{lemma}\label{lem:complexity-T4}
Given i.i.d. observations $X_1,...,X_n\sim N(0,\Sigma)$ and the function class $\mathcal{T}_4$ defined by (\ref{eq:net-structure-4}). Assume $\{S(\cdot,1),S(\cdot,0)\}$ is a regular proper scoring rule that satisfies Condition \ref{cond:G}, $\kappa_1\leq c$ for some sufficiently small constant $c>0$, and $\opnorm{\Sigma}\leq M=O(1)$. We have for any $\delta>0$,
\begin{eqnarray}
\sup_{T\in\mathcal{T}_4}\left|\frac{1}{n}\sum_{i=1}^n S(T(X_i),1)-\mathbb{E}S(T(X),1)\right|\leq C\kappa_1\kappa_2\left(\sqrt{\frac{p}{n}}+\sqrt{\frac{\log(1/\delta)}{n}}\right),
\end{eqnarray}
with probability at least $1-\delta$ for some universal constant $C>0$.
\end{lemma}

\begin{lemma}\label{lem:complexity-deep}
Given i.i.d. observations $X_1,...,X_n\sim\mathbb{P}$. Assume $\{S(\cdot,1),S(\cdot,0)\}$ is a regular proper scoring rule that satisfies Condition \ref{cond:G} and $\kappa\leq c$ for some sufficiently small constant $c>0$. We have for any $\delta>0$,
$$\sup_{T\in\mathcal{T}}\left|\frac{1}{n}\sum_{i=1}^n S(T(X_i),1)-\mathbb{E}S(T(X),1)\right|\leq C\kappa(2B)^{L-1}\left(\sqrt{\frac{p}{n}}+\sqrt{\frac{\log(1/\delta)}{n}}\right),$$
for both $\mathcal{T}=\mathcal{T}^L(\kappa,B)$ and $\mathcal{T}=\bar{\mathcal{T}}^L(\kappa,B)$ with probability at least $1-\delta$ for some universal constant $C>0$.
\end{lemma}

\subsection{Proofs of Proposition \ref{prop:cov-flat} and Proposition \ref{prop:T2-not-robust}}

\begin{proof}[Proof of Proposition \ref{prop:cov-flat}]
Define
$$F(w,u;\Sigma,\Gamma)=\mathbb{E}_{X\sim N(0,\Sigma)}\log T_{w,u}(X) + \mathbb{E}_{X\sim N(0,\Gamma)}\log(1-T_{w,u}(X)),$$
where $T_{w,u}(x)=\sig\left(\sum_{j\geq 1}w_j\sig(u_j^Tx)\right)$. Then, we have
$$F(\Sigma,\Gamma)=\max_{\|w\|_1\leq\kappa,u_j\in\mathbb{R}^p}F(w,u;\Sigma,\Gamma).$$
We calculate the gradient and Hessian of $F(w,u;\Sigma,\Gamma)$ with respect to $w$. To do this, we define $g_u(x)$ to be a vector with the same dimension as $w$ and each of its coordinate takes $\sig(u_j^Tx)$. With this notation, we can write $T_{w,u}(x)=\sig(w^Tg_u(x))$. By standard calculation, we get
\begin{eqnarray*}
\nabla_w F(w,u;\Sigma,\Gamma) &=& \mathbb{E}_{X\sim N(0,\Sigma)}\left(1-T_{w,u}(X)\right)g_u(X) - \mathbb{E}_{X\sim N(0,\Gamma)}T_{w,u}(X)g_u(X), \\
\nabla^2_w F(w,u;\Sigma,\Gamma) &=& -\mathbb{E}_{X\sim N(0,\Sigma)}T_{w,u}(X)\left(1-T_{w,u}(X)\right)g_u(X)g_u(X)^T \\
&& - \mathbb{E}_{X\sim N(0,\Gamma)}T_{w,u}(X)\left(1-T_{w,u}(X)\right)g_u(X)g_u(X)^T.
\end{eqnarray*}
For any $u_j\in\mathbb{R}^p$,
\begin{eqnarray*}
\mathbb{E}_{X\sim N(0,\Sigma)}\sig(u_j^TX) &=& \mathbb{E}_{X\sim N(0,\Sigma)}\sig(-u_j^TX) \\
&=& \mathbb{E}_{X\sim N(0,\Sigma)}(1-\sig(u_j^TX)),
\end{eqnarray*}
which immediately implies $\mathbb{E}_{X\sim N(0,\Sigma)}\sig(u_j^TX)=1/2$. By the same argument, we also have $\mathbb{E}_{X\sim N(0,\Gamma)}\sig(u_j^TX)=1/2$. Therefore,
$$\nabla_w F(w,u;\Sigma,\Gamma)|_{w=0}=\frac{1}{2}\mathbb{E}_{X\sim N(0,\Sigma)}g_u(X)-\frac{1}{2}\mathbb{E}_{X\sim N(0,\Gamma)}g_u(X)=0.$$
Moreover, since $-\nabla^2_w F(w,u;\Sigma,\Gamma)$ is positive semi-definite, $F(w,u;\Sigma,\Gamma)$ is a concave function in $w$. This implies
$$\max_{\|w\|_1\leq\kappa}F(w,u;\Sigma,\Gamma)=F(0,u;\Sigma,\Gamma)=-\log 4.$$
Taking maximum over $u$,  we have $F(\Sigma,\Gamma)=-\log 4$, regardless of the values of $\Sigma$ and $\Gamma$.
\end{proof}

\begin{proof}[Proof of Proposition \ref{prop:T2-not-robust}]
We first introduce some notation. We define
$$J(\gamma^2)=\max_{\|w\|_1\leq\kappa,u_j\in\mathbb{R}}F(w,u;\gamma^2),$$
where
$$F(w,u;\gamma^2)=\mathbb{E}_{X\sim (1-\epsilon)N(0,\sigma^2)+\epsilon N(0,\tau^2)}\log T_{w,u}(X)+\mathbb{E}_{X\sim N(0,\gamma^2)}\log(1-T_{w,u}(X)),$$
with $T_{w,u}(x)$ defined by $T_{w,u}(x)=\sig\left(\sum_{j\geq 1}w_j\relu(u_jx)\right)$. With these notation, we need to prove $J([(1-\epsilon)\sigma+\epsilon\tau]^2)=\min_{\gamma^2}J(\gamma^2)$.

The gradient and Hessian of $F(w,u;\gamma^2)$ with respect to $w$ are given by
\begin{eqnarray*}
\nabla_w F(w,u;\gamma^2) &=& \mathbb{E}_{X\sim (1-\epsilon)N(0,\sigma^2)+\epsilon N(0,\tau^2)}\left(1-T_{w,u}(X)\right)g_u(X) \\
&& - \mathbb{E}_{X\sim N(0,\gamma^2)}T_{w,u}(X)g_u(X), \\
\nabla^2_w F(w,u;\gamma^2) &=& -\mathbb{E}_{X\sim (1-\epsilon)N(0,\sigma^2)+\epsilon N(0,\tau^2)}T_{w,u}(X)\left(1-T_{w,u}(X)\right)g_u(X)g_u(X)^T \\
&& - \mathbb{E}_{X\sim N(0,\gamma^2)}T_{w,u}(X)\left(1-T_{w,u}(X)\right)g_u(X)g_u(X)^T,
\end{eqnarray*}
where we use $g_u(X)$ to denote the vector whose $j$th coordinate is $\relu(u_jX)$. It is not hard to see that
$$\mathbb{E}_{X\sim N(0,\gamma^2)}\relu(u_jX)=|u_j|\gamma\sqrt{\frac{2}{\pi}},$$
and
$$\mathbb{E}_{X\sim (1-\epsilon)N(0,\sigma^2)+\epsilon N(0,\tau^2)}\relu(u_jX)=|u_j|[(1-\epsilon)\sigma+\epsilon\tau]\sqrt{\frac{2}{\pi}}.$$
Therefore, we have the identity
$$\mathbb{E}_{X\sim N(0,[(1-\epsilon)\sigma+\epsilon\tau]^2)}\relu(u_jX)=\mathbb{E}_{X\sim (1-\epsilon)N(0,\sigma^2)+\epsilon N(0,\tau^2)}\relu(u_jX),$$
which is equivalent to $\nabla_w F(w,u;[(1-\epsilon)\sigma+\epsilon\tau]^2)|_{w=0}=0$. Moreover, since $-\nabla^2_w F(w,u;[(1-\epsilon)\sigma+\epsilon\tau]^2)$ is positive semi-definite, $F(w,u;[(1-\epsilon)\sigma+\epsilon\tau]^2)$ is a concave function in $w$. This implies
$$\max_{\|w\|_1\leq\kappa}F(w,u;[(1-\epsilon)\sigma+\epsilon\tau]^2)=F(0,u;[(1-\epsilon)\sigma+\epsilon\tau]^2)=-\log 4.$$
Taking maximum over $u$, we have $J([(1-\epsilon)\sigma+\epsilon\tau]^2)=-\log 4$.
For any $\gamma^2$, $J(\gamma^2)\geq F(0,u;\gamma^2)=-\log 4$. Hence, we have the desired conclusion that $J([(1-\epsilon)\sigma+\epsilon\tau]^2)=\min_{\gamma^2}J(\gamma^2)$.
\end{proof}

\subsection{Proofs of Proposition \ref{prop:robust-mean} and Proposition \ref{prop:cov-no-con}}

\begin{proof}[Proof of Proposition \ref{prop:robust-mean}]
We use the notation $F_{w,u}(P,Q)=\mathbb{E}_{X\sim P}\log T_{w,u}(X) +\mathbb{E}_{X\sim Q}\log(1-T_{w,u}(X))$, with $T_{w,u}(x)=\sig\left(\sum_{j\geq 1}w_j\sig(u_j^Tx)\right)$. Thus, we can write
\begin{equation}
\wh{\theta}=\argmin_{\eta\in\mathbb{R}^p}F(\mathbb{P}_n,N(\eta,I_p)), \label{eq:JS-mean-def}
\end{equation}
where $F(P,Q)=\max_{\|w\|_1\leq \kappa,u_j\in\mathbb{R}^p}F_{w,u}(P,Q)$, and $\mathbb{P}_n=\frac{1}{n}\sum_{i=1}^n\delta_{X_i}$ is the empirical measure. Let $P$ be the data generating process that satisfies $\TV(P,N(\theta,I_p))\leq\epsilon$, and then there exist probability distributions $Q_1$ and $Q_2$, such that
$$P+\epsilon Q_1=N(\theta,I_p)+\epsilon Q_2.$$
The explicit construction of $Q_1,Q_2$ is given in the proof of Theorem 5.1 of \cite{chen2018robust}. This implies that
\begin{eqnarray}
\nonumber && |F(P,N(\eta,I_p))-F(N(\theta,I_p),N(\eta,I_p))| \\
\nonumber &\leq& \sup_{\|w\|_1\leq\kappa,u}|F_{w,u}(P,N(\eta,I_p))-F_{w,u}(N(\theta,I_p),N(\eta,I_p))| \\
\nonumber &=& \epsilon\sup_{\|w\|_1\leq\kappa,u}|\mathbb{E}_{X\sim Q_2}\log(2T_{w,u}(X))-\mathbb{E}_{X\sim Q_1}\log(2T_{w,u}(X))| \\
\label{eq:JS-stability} &\leq& 2\kappa\epsilon.
\end{eqnarray}
Then,
\begin{eqnarray}
\label{eq:1M} F(N(\theta,I_p),N(\wh{\theta},I_p)) &\leq& F(P,N(\wh{\theta},I_p)) + 2\kappa\epsilon \\
\label{eq:M2} &\leq& F(\mathbb{P}_n,N(\wh{\theta},I_p)) + C\kappa\left(\sqrt{\frac{p}{n}}+\sqrt{\frac{\log(1/\delta)}{n}}\right) + 2\kappa\epsilon \\
\label{eq:M3} &\leq& F(\mathbb{P}_n,N(\theta,I_p)) + C\kappa\left(\sqrt{\frac{p}{n}}+\sqrt{\frac{\log(1/\delta)}{n}}\right) + 2\kappa\epsilon \\
\label{eq:M4} &\leq& F(P,N(\theta,I_p)) + 2C\kappa\left(\sqrt{\frac{p}{n}}+\sqrt{\frac{\log(1/\delta)}{n}}\right) + 2\kappa\epsilon \\
\label{eq:M5} &\leq& F(N(\theta,I_p),N(\theta,I_p)) + 2C\kappa\left(\sqrt{\frac{p}{n}}+\sqrt{\frac{\log(1/\delta)}{n}}\right) + 4\kappa\epsilon.
\end{eqnarray}
The inequalities (\ref{eq:1M}) and (\ref{eq:M5}) are direct consequences of (\ref{eq:JS-stability}). We have used Lemma \ref{lem:complexity-T1} for (\ref{eq:M2}) and (\ref{eq:M4}), and (\ref{eq:M3}) is implied by the definition of the estimator (\ref{eq:JS-mean-def}). By the definition of $F(N(\theta,I_p),N(\wh{\theta},I_p))-F(N(\theta,I_p),N(\theta,I_p))$, we obtain the following inequality that
$$F_{w,u}(N(\theta,I_p),N(\wh{\theta},I_p))+\log 4\leq 2C\kappa\left(\sqrt{\frac{p}{n}}+\sqrt{\frac{\log(1/\delta)}{n}}\right) + 4\kappa\epsilon,$$
uniformly over all $\|w\|_1\leq\kappa$ and all $u$ with probability at least $1-\delta$. Choose $w_1=\kappa$, $w_j=0$ for all $j\geq 2$ and $u_1=v$ for some unit vector $\|v\|=1$, and then we have
\begin{equation}
f(\kappa;v^T\theta,v^T\wh{\theta}) \leq 2C\kappa\left(\sqrt{\frac{p}{n}}+\sqrt{\frac{\log(1/\delta)}{n}}\right) + 4\kappa\epsilon,\label{eq:bound-t-delta}
\end{equation}
where
$$f(t;\delta_1,\delta_2) = \mathbb{E}\log\frac{2}{1+e^{-t\sig(Z+\delta_1)}} + \mathbb{E}\log\frac{2}{1+e^{t\sig(Z+\delta_2)}},$$
with $Z\sim N(0,1)$. Direct calculations give
\begin{eqnarray*}
\frac{\partial}{\partial t}f(t;\delta_1,\delta_2) &=& \mathbb{E}\frac{1}{1+e^{t\sig(Z+\delta_1)}}\sig(Z+\delta_1) \\
&& - \mathbb{E}\frac{1}{1+e^{-t\sig(Z+\delta_2)}}\sig(Z+\delta_2), \\
\frac{\partial^2}{\partial t^2}f(t;\delta_1,\delta_2) &=& -\mathbb{E}\frac{e^{t\sig(Z+\delta_1)}}{(1+e^{t\sig(Z+\delta_1)})^2}|\sig(Z+\delta_1)|^2 \\
&& -\mathbb{E}\frac{e^{t\sig(Z+\delta_2)}}{(1+e^{t\sig(Z+\delta_2)})^2}|\sig(Z+\delta_2)|^2.
\end{eqnarray*}
Therefore, we have $f(0;\delta_1,\delta_2)=0$, $\frac{\partial}{\partial t}f(t;\delta_1,\delta_2)|_{t=0}=\frac{1}{2}\left(\mathbb{E}\sig(Z+\delta_1)-\mathbb{E}\sig(Z+\delta_2)\right)$, and $\frac{\partial^2}{\partial t^2}f(t;\delta_1,\delta_2)\geq -\frac{1}{2}$, which then implies
$$f(\kappa;\delta_1,\delta_2)\geq f(0;\delta_1,\delta_2) + \kappa\frac{\partial}{\partial t}f(t;\delta_1,\delta_2)|_{t=0}-\frac{1}{4}\kappa^2.$$
or equivalently $\kappa\frac{\partial}{\partial t}f(t;\delta_1,\delta_2)|_{t=0}\leq f(\kappa;\delta_1,\delta_2)+\frac{1}{4}\kappa^2$. In view of the bound (\ref{eq:bound-t-delta}), we have
\begin{eqnarray*}
&& \frac{\kappa}{2}\int\left(\sig(z+v^T{\theta})-\sig(z+v^T\wh{\theta})\right)\phi(z)dz \\
&\leq& 2C\kappa\left(\sqrt{\frac{p}{n}}+\sqrt{\frac{\log(1/\delta)}{n}}\right) + 4\kappa\epsilon + \frac{1}{4}\kappa^2,
\end{eqnarray*}
where $\phi(z)=\frac{1}{\sqrt{2\pi}}e^{-z^2/2}$.
A symmetric argument with $u_1=-v$ leads to the same bound for $\frac{\kappa}{2}\int\left(\sig(z+v^T\wh{\theta})-\sig(z+v^T{\theta})\right)\phi(z)dz$, and thus
\begin{eqnarray*}
&& \left|\int\left(\sig(z+v^T\wh{\theta})-\sig(z+v^T\theta)\right)\phi(z)dz\right| \\
&\leq& 2C\left(\sqrt{\frac{p}{n}}+\sqrt{\frac{\log(1/\delta)}{n}}\right) + 4\epsilon + \frac{1}{4}\kappa.
\end{eqnarray*}
Define the following function $h(t)=\int \sig(z+v^T\theta+t)\phi(z)dz$, take $\kappa=O(\sqrt{p/n}+\epsilon)$ and $\delta=e^{C'(p+n\epsilon^2)}$, and then the above bound becomes $|h(v^T(\wh{\theta}-\theta))-h(0)|=O(\sqrt{p/n}+\epsilon)$. It is easy to see that $h'(0)\geq \int \sig(z+M)(1-\sig(z+M))\phi(z)dz$, which is a constant, and the continuity of $h'(t)$ implies that there are small constants $c_1,c_2>0$, such that $\inf_{|t|\leq c_1}|h'(t)|\geq c_2$. Thus, as long as $|h(t)-h(0)|$ is sufficiently small, we have $|h(t)-h(0)|\geq c_2|t|$, which implies that $|v^T(\wh{\theta}-\theta)|=O(\sqrt{p/n}+\epsilon)$. The proof is complete by taking supreme over all unit vector $v$.
\end{proof}

\begin{proof}[Proof of Proposition \ref{prop:cov-no-con}]
We use the same notation $F_{w,u}(P,Q)$ defined in the proof of Proposition \ref{prop:robust-mean}, but $T_{w,u}(x)=\sig(\sum_{j\geq 1}w_j\relu(u_j^Tx))$. Then, by the same argument used in (\ref{eq:1M})-(\ref{eq:M5}) (with $\epsilon=0$ and Lemma \ref{lem:complexity-T1} replaced by Lemma \ref{lem:complexity-T2}), we have
$$F_{w,u}(N(0,\Sigma),N(0,\wh{\Sigma}))+\log 4\leq 2C\kappa\left(\sqrt{\frac{p}{n}}+\sqrt{\frac{\log(1/\delta)}{n}}\right),$$
uniformly over all $\|w\|_1\leq\kappa$ and all $\|u_j\|\leq 1$ with probability at least $1-\delta$. Choose $w_1=w_2=\kappa/2$, $w_j=0$ for all $j\geq 3$, $u_1=v$ and $u_2=-v$ for some unit vector $\|v\|=1$, and then we have
\begin{equation}
f\left(\kappa;\sqrt{v^T\Sigma v},\sqrt{v^T\wh{\Sigma}v}\right) \leq 2C\kappa\left(\sqrt{\frac{p}{n}}+\sqrt{\frac{\log(1/\delta)}{n}}\right), \label{eq:911gt3}
\end{equation}
where
$$f(t;\delta_1,\delta_2) = \mathbb{E}\log\frac{2}{1+e^{-\frac{t}{2}\delta_1|Z|}}+\mathbb{E}\log\frac{2}{1+e^{\frac{t}{2}\delta_2|Z|}},$$
with $Z\sim N(0,1)$. Then,
\begin{eqnarray*}
\frac{\partial}{\partial t}f(t;\delta_1,\delta_2) &=& \frac{1}{2}\mathbb{E}\frac{1}{1+e^{t\delta_1|Z|/2}}\delta_1|Z| - \frac{1}{2}\mathbb{E}\frac{1}{1+e^{-t\delta_2|Z|/2}}\delta_2|Z|, \\
\frac{\partial^2}{\partial t^2}f(t;\delta_1,\delta_2) &=& -\frac{1}{4}\mathbb{E}\frac{e^{t\delta_1|Z|/2}}{(1+e^{t\delta_1|Z|/2})^2}|\delta_1Z|^2 - \frac{1}{4}\mathbb{E}\frac{e^{t\delta_2|Z|/2}}{(1+e^{t\delta_2|Z|/2})^2}|\delta_2Z|^2.
\end{eqnarray*}
Therefore, we have $f(0;\delta_1,\delta_2)=0$, $\frac{\partial}{\partial t}f(t;\delta_1,\delta_2)|_{t=0}=\frac{1}{4}\mathbb{E}|Z|(\delta_1-\delta_2)$, and
$$\frac{\partial^2}{\partial t^2}f\left(t;\sqrt{v^T\Sigma v},\sqrt{v^T\wh{\Sigma}v}\right)\geq-\frac{M}{8},$$
which implies
$$f(\kappa;\delta_1,\delta_2)\geq f(0;\delta_1,\delta_2) + \kappa\frac{\partial}{\partial t}f(t;\delta_1,\delta_2)|_{t=0}-\frac{M}{16}\kappa^2.$$
Then, by the bound (\ref{eq:911gt3}), we have
$\sqrt{v^T\Sigma v}-\sqrt{v^T\wh{\Sigma}v}\lesssim \sqrt{\frac{p}{n}}$
by taking $\kappa=O(\sqrt{p/n})$ and $\delta=e^{C'p}$. A symmetric argument also leads to the same bound for $\sqrt{v^T\wh{\Sigma}v}-\sqrt{v^T\Sigma v}$, so that $\left|\sqrt{v^T\wh{\Sigma}v}-\sqrt{v^T\Sigma v}\right|\lesssim \sqrt{\frac{p}{n}}$, which is equivalent to $|v^T(\wh{\Sigma}-\Sigma)v|\lesssim \left(\sqrt{v^T\wh{\Sigma}v}+\sqrt{v^T\Sigma v}\right)\sqrt{\frac{p}{n}}\leq 2\sqrt{M}\sqrt{\frac{p}{n}}$. The proof is complete by taking supreme over all unit vector $v$.
\end{proof}

\subsection{Proofs of Theorem \ref{thm:cov-T3} and Theorem \ref{thm:cov-T4}}

\begin{proof}[Proof of Theorem \ref{thm:cov-T3}]
We use the notation
$$F_{w,u,b}(P,Q)=\mathbb{E}_{X\sim P}S(T_{w,u,b}(X),1)+\mathbb{E}_{X\sim Q}S(T_{w,u,b}(X),0),$$
with $T_{w,u,b}(x)=\sig(\sum_{j\geq 1}w_j\sig(u_j^Tx+b_j))$. We also write
$$F(P,Q)=\max_{\|w\|_1\leq\kappa,u_j\in\mathbb{R}^p,b_j\in\mathbb{R}}F_{w,u,b}(P,Q).$$
For $P$ that satisfies $\TV(P,N(0,\Sigma))\leq\epsilon$, there exist probability distributions $Q_1$ and $Q_2$, such that $P+\epsilon Q_1=N(0,\Sigma)+\epsilon Q_2$. Therefore,
\begin{eqnarray*}
&& \left|F(P,N(0,\Gamma))-F(N(0,\Sigma),N(0,\Gamma))\right| \\
&\leq& \max_{\|w\|_1\leq\kappa,u_j\in\mathbb{R}^p,b_j\in\mathbb{R}}\left|F_{w,u,b}(P,N(0,\Gamma))-F_{w,u,b}(N(0,\Sigma),N(0,\Gamma))\right| \\
&=& \epsilon\max_{\|w\|_1\leq\kappa,u_j\in\mathbb{R}^p,b_j\in\mathbb{R}}\left|\mathbb{E}_{X\sim Q_2}S(T_{w,u,b}(X),1)-\mathbb{E}_{X\sim Q_1}S(T_{w,u,b}(X),1)\right| \\
&\leq& 2\epsilon\sup_{|t-1/2|\leq\kappa}\left|S(t,1)-G\left(\frac{1}{2}\right)-\frac{1}{2}G'\left(\frac{1}{2}\right)\right| \\
&\leq& 2\epsilon\kappa\sup_{\left|t-\frac{1}{2}\right|\leq\kappa}\left|(1-t)G''(t)\right| \leq 2C_1\epsilon\kappa,
\end{eqnarray*}
where we have used $\sup_{\left|t-\frac{1}{2}\right|\leq\kappa}\left|(1-t)G''(t)\right|\leq C_1$ because of the smoothness of $G(t)$ at $t=1/2$ by Condition \ref{cond:G}.
In the second last inequality above, we have used the fact that $\frac{\partial}{\partial t}S(t,1)=(1-t)G''(t)$ and $S(t,1)=G(t)+(1-t)G'(t)$.

 Then, by the same argument used in (\ref{eq:1M})-(\ref{eq:M5}) (with Lemma \ref{lem:complexity-T1} replaced by Lemma \ref{lem:complexity-T3}), we have
\begin{equation}
F_{w,u,b}(N(0,\Sigma),N(0,\wh{\Sigma}))-2G(1/2)\leq 2C\kappa\left(\sqrt{\frac{p}{n}}+\sqrt{\frac{\log(1/\delta)}{n}}\right)+4C_1\epsilon\kappa,\label{eq:rs7}
\end{equation}
uniformly over all $\|w\|_1\leq\kappa$, all $u_j$ and all $b_j$ with probability at least $1-\delta$. We choose $w_1=\kappa$, $w_j=0$ for all $j\geq 2$, $u_1=v/\sqrt{v^T\Sigma v}$ for some unit vector $v$, and $b_1=1$. Then, we have
\begin{equation}
f\left(\kappa;\sqrt{\frac{v^T\wh{\Sigma}v}{v^T\Sigma v}}\right) \leq 2C\kappa\left(\sqrt{\frac{p}{n}}+\sqrt{\frac{\log(1/\delta)}{n}}\right)+4C_1\epsilon\kappa, \label{eq:routine}
\end{equation}
where
$$f(t;\Delta) = \mathbb{E}S\left(\frac{1}{1+e^{-t\sig(Z-1)}},1\right)+\mathbb{E}S\left(\frac{1}{1+e^{-t\sig(\Delta Z-1)}},0\right)-2G(1/2),$$
with $Z\sim N(0,1)$. We introduce some polynomials,
$$p(t)=(1-t)^2t,\quad \bar{p}(t)=2(1-t)^2t^2-(1-t)^3t,\quad \tilde{p}(t)=(1-t)^3t^2,$$
$$q(t)=(1-t)t^2,\quad \bar{q}(t)=2(1-t)^2t^2-(1-t)t^3,\quad \tilde{q}(t)=(1-t)^2t^3.$$
Then, by standard calculations, we get
\begin{eqnarray*}
\frac{\partial}{\partial t}f(t;\Delta)
 &=& \mathbb{E}p\left(L(t,1,Z)\right)G''\left(L(t,1,Z)\right)\sig(Z-1) \\
&& - \mathbb{E}q\left(L(t,\Delta,Z)\right)G''\left(L(t,\Delta,Z)\right)\sig(\Delta Z-1), \\
\end{eqnarray*}
and
\begin{eqnarray*}
&& \frac{\partial^2}{\partial t^2}f(t;\Delta) \\
&=& -\mathbb{E}\left[\bar{p}(L(t,1,Z))G''(L(t,1,Z)) - \tilde{p}(L(t,1,Z))G'''(L(t,1,Z))\right]|\sig(Z-1)|^2 \\
&& - \mathbb{E}\left[\wt{q}(L(t,\Delta,Z))G''(L(t,\Delta,Z))-\tilde{q}(L(t,\Delta,Z))G'''(L(t,\Delta,Z))\right]|\sig(\Delta Z-1)|^2,
\end{eqnarray*}
where we use the notation $L(t,\delta,Z)=\sig(t\sig(\delta Z-1))$. Note that $f(0;\Delta)=0$ and
$$\frac{\partial}{\partial t}f(t;\Delta)|_{t=0}=\frac{1}{8}G''(1/2)\mathbb{E}\left(\sig(Z-1)-\sig(\Delta Z-1)\right).$$
Since
$$\bar{p}(1/2)G''(1/2)-\tilde{p}(1/2)G'''(1/2)=\frac{1}{16}G''(1/2)-\frac{1}{32}G'''(1/2)\geq \frac{1}{32}c_0$$
by Condition \ref{cond:G}, for a sufficiently small $\kappa$, we have
$$\inf_{|t-1/2|\leq\kappa}\left[\bar{p}(t)G''(t)-\tilde{p}(t)G'''(t)\right]>0.$$
Moreover, there is some constant $C_2>0$, such that
$$\sup_{|t-1/2|\leq\kappa}\left[\bar{p}(t)G''(t)-\tilde{p}(t)G'''(t)\right]\leq C_2.$$
For the same reason, we also have
$$0<\inf_{|t-1/2|\leq\kappa}\left[\bar{q}(t)G''(t)-\tilde{q}(t)G'''(t)\right]\leq \sup_{|t-1/2|\leq\kappa}\left[\bar{q}(t)G''(t)-\tilde{q}(t)G'''(t)\right]\leq C_2.$$
Since $|t|\leq\kappa$ implies $|L(t,\delta,Z)-1/2|\leq\kappa$, we have
$$\inf_{|t|\leq\kappa}\frac{\partial^2}{\partial t^2}f(t;\Delta)\geq -C_2\mathbb{E}|\sig(Z-1)|^2-C_2\mathbb{E}|\sig(\Delta Z-1)|^2\geq -2C_2.$$
Therefore,
$$f(\kappa;\Delta)\geq f(0;\Delta) + \kappa\frac{\partial}{\partial t}f(t;\Delta)|_{t=0}-C_2\kappa^2.$$
Together with the bound (\ref{eq:routine}), we have
\begin{eqnarray*}
&& \frac{1}{8}G''(1/2)\mathbb{E}\left(\sig(Z-1)-\sig(\Delta Z-1)\right) \\
&\leq& 2C\left(\sqrt{\frac{p}{n}}+\sqrt{\frac{\log(1/\delta)}{n}}\right)+4C_1\kappa + C_2\kappa,
\end{eqnarray*}
with $\Delta=\sqrt{\frac{v^T\wh{\Sigma}v}{v^T\Sigma v}}$.
A symmetric argument with $w_1=-\kappa$ leads to the same bound for
$$\frac{1}{8}G''(1/2)\mathbb{E}\left(\sig(\Delta Z-1)-\sig(Z-1)\right).$$
Thus, with the choice $\kappa=O\left(\sqrt{\frac{p}{n}}+\epsilon\right)$ and $\delta=e^{-C'(p+n\epsilon^2)}$, we have
$$\left|h\left(\Delta\right)-h(1)\right|\leq C_3\left(\sqrt{\frac{p}{n}}+\epsilon\right),$$
where $h(t)=\int\sig(t z-1)(2\pi)^{-1/2}e^{-z^2/2}dz$. Note that
$$|h'(1)|=\left|\int \sig(z-1)(1-\sig(z-1))z(2\pi)^{-1/2}e^{-z^2/2}dz\right|$$
is a constant. The continuity of $h'(t)$ implies that there are small constants $c_1,c_2>0$, such that $\inf_{|t-1|\leq c_1}|h'(t)|\geq c_2$. Thus, as long as $|h(t)-h(1)|$ is sufficiently small, we have $|h(t)-h(1)|\geq c_2|t-1|$, which implies that $\left|\sqrt{v^T\wh{\Sigma}v}-\sqrt{v^T\Sigma v}\right|\lesssim \sqrt{v^T\Sigma v}\left(\sqrt{\frac{p}{n}}+\epsilon\right)\lesssim \sqrt{\frac{p}{n}}+\epsilon$. Following the last several lines of the proof of Proposition \ref{prop:cov-no-con}, we obtain the desired result.
\end{proof}

\begin{proof}[Proof of Theorem \ref{thm:cov-T4}]
We use the notation $T_{w,v,u}(x) = \sig\left(\sum_{j\geq 1}w_j\sig\left(\sum_{l=1}^Hv_{jl}\relu(u_l^Tx)\right)\right)$, and then define
$$F_{w,v,u}(P,Q)=\mathbb{E}_{X\sim P}S(T_{w,v,u}(X),1)+\mathbb{E}_{X\sim Q}S(T_{w,v,u}(X),0).$$
Using Lemma \ref{lem:complexity-T4} and following the same argument that leads to (\ref{eq:rs7}), we have
$$F_{w,v,u}(N(0,\Sigma),N(0,\wh{\Sigma}))-2G(1/2)\leq 2C\kappa_1\kappa_2\left(\sqrt{\frac{p}{n}}+\sqrt{\frac{\log(1/\delta)}{n}}\right)+4C_1\epsilon\kappa_1,$$
uniformly over all $\|w\|_1\leq\kappa_1$, all $\|v_j\|_1\leq\kappa_2$, and all $\|u_l\|\leq 1$. Choose $w_1=\kappa_1$, $w_j=0$ for all $j\geq 2$, $v_{11}=v_{12}=\frac{1}{2}$, $v_{1l}=0$ for all $l\geq 3$, and $u_1=-u_2=v$ for a unit vector $v$. Then, we have
$$f\left(\kappa_1; \sqrt{v^T\Sigma v},\sqrt{v^T\wh{\Sigma}v}\right)\leq 2C\kappa_1\kappa_2\left(\sqrt{\frac{p}{n}}+\sqrt{\frac{\log(1/\delta)}{n}}\right)+4C_1\epsilon\kappa_1,$$
where
$$f(t;\delta_1,\delta_2)=\mathbb{E}S\left(\frac{1}{1+e^{-t\sig(\delta_1|Z|/2)}},1\right)+\mathbb{E}S\left(\frac{1}{1+e^{-t\sig(\delta_2|Z|/2)}},0\right)-2G(1/2).$$
Then, using the same argument in the proof of Theorem \ref{thm:cov-T3}, we have
$$|h(\delta_1)-h(\delta_2)|\lesssim \sqrt{\frac{p}{n}}+\epsilon,$$
with the choice $\kappa_1=O\left(\sqrt{\frac{p}{n}}+\epsilon\right)$, $\kappa_2\asymp 1$, and $\delta=e^{-C'(p+n\epsilon^2)}$, where $\delta_1=\sqrt{v^T\Sigma v}$, $\delta_2=\sqrt{v^T\wh{\Sigma} v}$, and $h(t)=\int \sig(t|z|/2)\phi(z)dz$. The notation $\phi(\cdot)$ is used for the density function of $N(0,1)$. Since
\begin{eqnarray*}
h'(t) &=& \int\sig(t|z|/2)(1-\sig(t|z|/2))|z|\phi(z)dz \\
&\geq& \int\sig(M^{1/2}|z|/2)(1-\sig(M^{1/2}|z|/2))|z|\phi(z)dz,
\end{eqnarray*}
which means $h'(t)$ is lower bounded by a constant uniformly over $|t|\leq M^{1/2}$, we have $|h(\delta_1)-h(\delta_2)|\geq c|\delta_1-\delta_2|$. Following the last several lines of the proof of Proposition \ref{prop:cov-no-con}, we obtain the desired result.
\end{proof}

\subsection{Proofs of Theorem \ref{thm:mean-cov} and Theorem \ref{thm:ellip}}

\begin{proof}[Proof of Theorem \ref{thm:mean-cov}]
We use the notation $T_{w,g}(x)=\sig\left(\sum_{j\geq 1}w_jg_j(x)\right)$, and write
$$F_{w,g}(P,Q)=\mathbb{E}_{X\sim P}S(T_{w,g}(X),1)+\mathbb{E}_{X\sim Q}S(T_{w,g}(X),0).$$
By Lemma \ref{lem:complexity-deep} and the same argument that leads to (\ref{eq:rs7}), we have
\begin{equation}
F_{w,g}(N(\theta,\Sigma),N(\wh{\theta},\wh{\Sigma})) - 2G(1/2) \leq 2C\kappa(2B)^{L-1}\left(\sqrt{\frac{p}{n}}+\sqrt{\frac{\log(1/\delta)}{n}}\right) + 4C_1B^{L-1}\kappa\epsilon,\label{eq:amg-gtr}
\end{equation}
uniformly over $\|w\|_1\leq\kappa$ and $g_j\in\mathcal{G}^L(B)$. The second term in the above bound is $4C_1B^{L-1}\kappa\epsilon$ instead of $4C_1\kappa\epsilon$ in (\ref{eq:rs7}) because $T\in\mathcal{T}^L(\kappa,B)$ implies that $\sup_x|T(x)-1/2|\leq B^{L-1}\kappa$ according to the proof of Lemma \ref{lem:complexity-deep}.

We need to show that the function $\sig(u^Tx+b)\in\mathcal{G}^L(B)$ for any $u\in\mathbb{R}^p$ and any $b\in\mathbb{R}$. Note that $\sig(u^Tx+b)\in\mathcal{G}^1(B)$ is obvious. This is also true for $\mathcal{G}^2(B)$ by taking $\sig(u^Tx+b)\in\mathcal{G}_{\sig}$, $v_1=1$ and $v_h=0$ for all $h\geq 2$, because $\relu(\sig(\cdot))=\sig(\cdot)$. Suppose $\sig(u^Tx+b)\in\mathcal{G}^l(B)$, we also have $\sig(u^Tx+b)\in\mathcal{G}^{l+1}(B)$, because $\relu(\relu(\cdot))=\relu(\cdot)$, and the claim is proved by an induction argument.

Choose $w_1=\kappa$, $w_j=0$ for all $j\geq 2$, and $g_1(x)=\sig(u^Tx+b)$, and then we have
\begin{eqnarray}
\nonumber && \mathbb{E}_{X\sim N(\theta,\Sigma)}S\left(\frac{1}{1+e^{-\kappa\sig(u^TX+b)}},1\right) + \mathbb{E}_{X\sim N(\wh{\theta},\wh{\Sigma})}S\left(\frac{1}{1+e^{-\kappa\sig(u^TX+b)}},0\right)  \\
\label{eq:gt2-rs} &\leq& 2G(1/2) + 2C\kappa(2B)^{L-1}\left(\sqrt{\frac{p}{n}}+\sqrt{\frac{\log(1/\delta)}{n}}\right) + 4C_1B^{L-1}\kappa\epsilon,
\end{eqnarray}
uniformly over all $u\in\mathbb{R}^p$ and all $b\in\mathbb{R}$, with probability at least $1-\delta$.

We further specify that $u=v$ for some unit vector $v$ and $b=-v^T\theta$. Then, the inequality (\ref{eq:gt2-rs}) becomes
$$f\left(\kappa;\sqrt{v^T\Sigma v},\sqrt{v^T\wh{\Sigma}v},v^T(\wh{\theta}-\theta)\right)\leq 2C\kappa(2B)^{L-1}\left(\sqrt{\frac{p}{n}}+\sqrt{\frac{\log(1/\delta)}{n}}\right) + 4C_1B^{L-1}\kappa\epsilon,$$
where
$$f(t;\delta_1,\delta_2,\Delta) = \mathbb{E}S\left(\frac{1}{1+e^{-t\sig(\delta_1 Z)}},1\right) + \mathbb{E}S\left(\frac{1}{1+e^{-t\sig(\delta_2 Z+\Delta)}},0\right)-2G(1/2),$$
with $Z\sim N(0,1)$. Then, using the same argument in the proof of Theorem \ref{thm:cov-T3}, we have
\begin{equation}
\left|\int \sig(\delta_1 z)\phi(z)dz - \int \sig(\delta_2 z+\Delta)\phi(z)dz\right|\lesssim \sqrt{\frac{p}{n}}+\epsilon,\label{eq:cayman-gts}
\end{equation}
with the choice $\kappa=O\left(\sqrt{\frac{p}{n}}+\epsilon\right)$, $B\asymp 1$, $L\asymp 1$ and $\delta=e^{-C'(p+n\epsilon^2)}$, where $\delta_1=\sqrt{v^T\Sigma v}$, $\delta_2=\sqrt{v^T\wh{\Sigma} v}$, $\Delta=v^T(\wh{\theta}-\theta)$, and $\phi(\cdot)$ is the density function of $N(0,1)$. Since $\int \sig(\delta_1 z)\phi(z)dz=1/2=\int \sig(\delta_2 z)\phi(z)dz$, this is equivalent to the bound $|h(0)-h(\Delta)|\lesssim \sqrt{\frac{p}{n}}+\epsilon$ with $h(t)=\int \sig(\delta_2 z+t)\phi(z)dz$. It is easy to see that $|h'(0)|\geq \inf_{|\delta_2|\leq M^{1/2}}\int \sig(\delta_2z)(1-\sig(\delta_2z))\phi(z)dz$, which is a constant, and the continuity of $h'(t)$ implies that there are small constants $c_1,c_2>0$, such that $\inf_{|t|\leq c_1}|h'(t)|\geq c_2$. Thus, as long as $|h(t)-h(0)|$ is sufficiently small, we have $|h(t)-h(0)|\geq c_2|t|$, which implies that $|v^T(\wh{\theta}-\theta)|\lesssim \sqrt{p/n}+\epsilon$. Taking supreme over all unit vector, we have $\|\wh{\theta}-\theta\|\lesssim\sqrt{p/n}+\epsilon$ with high probability.

To show the error bound for the covariance matrix estimator, we choose $u=v$ and $b=-v^T\theta-1$ for some unit vector $v$ in the inequality (\ref{eq:gt2-rs}). Then, (\ref{eq:cayman-gts}) becomes
$$\left|\int \sig(\delta_1 z-1)\phi(z)dz - \int \sig(\delta_2 z+\Delta-1)\phi(z)dz\right|\lesssim \sqrt{\frac{p}{n}}+\epsilon.$$
Since $|\sig(\delta_2 z+\Delta-1) - \sig(\delta_2 z-1)|\leq |\Delta|\lesssim \sqrt{\frac{p}{n}}+\epsilon$, we have
$$\left|\int \sig(\delta_1 z-1)\phi(z)dz - \int \sig(\delta_2 z-1)\phi(z)dz\right|\lesssim \sqrt{\frac{p}{n}}+\epsilon,$$
which can be written as $|\bar{h}(\delta_1)-\bar{h}(\delta_2)|\lesssim \sqrt{\frac{p}{n}}+\epsilon$, with $\bar{h}(t)=\int \sig(tz-1)\phi(z)dz$. Since $\inf_{0<t\leq M^{1/2}}|\bar{h}'(t)|=\inf_{0<t\leq M^{1/2}}\left|\int\sig(tz-1)(1-\sig(tz-1))z\phi(z)dz\right|$ is lower bounded by some constant, we have $|\delta_1-\delta_2|\lesssim \sqrt{\frac{p}{n}}+\epsilon$, and by following the last several lines of the proof of Proposition \ref{prop:cov-no-con}, we obtain the desired result.
\end{proof}

\begin{proof}[Proof of Theorem \ref{thm:ellip}]
We use the same notation $T_{w,g}(x)$ and $F_{w,g}(P,Q)$ defined in the proof of Theorem \ref{thm:mean-cov}. The same argument that leads to (\ref{eq:amg-gtr}) gives the following inequality,
\begin{eqnarray}
\nonumber && F_{w,g}(E(\theta,\Sigma,H),E(\wh{\theta},\wh{\Sigma},\wh{H})) - 2G(1/2) \\
\label{eq:focus-rs} &\leq& 2C\kappa(2B)^{L-1}\left(\sqrt{\frac{p}{n}}+\sqrt{\frac{\log(1/\delta)}{n}}\right) + 4C_1B^{L-1}\kappa\epsilon,
\end{eqnarray}
uniformly over $\|w\|_1\leq\kappa$ and $g_j\in\mathcal{G}^L(B)$.

Choose $w_1=\kappa$, $w_j=0$ for all $j\geq 2$, and $g_1(x)=\ramp\left(\frac{v^Tx-v^T\wh{\theta}}{\sqrt{v^T\Sigma v}}\right)$, for some unit vector $v$, and then we have
$$f\left(\kappa;\sqrt{v^T\Sigma v},\sqrt{v^T\wh{\Sigma}v},v^T(\theta-\wh{\theta})\right)\leq 2C\kappa(2B)^{L-1}\left(\sqrt{\frac{p}{n}}+\sqrt{\frac{\log(1/\delta)}{n}}\right) + 4C_1B^{L-1}\kappa\epsilon,$$
where
\begin{eqnarray*}
f(t;\delta_1,\delta_2,\Delta) &=& \int S\left(\frac{1}{1+e^{-t\ramp(z+\Delta/\delta_1)}},1\right)dH(t) \\
&& + \int S\left(\frac{1}{1+e^{-t\ramp((\delta_2/\delta_1) z)}},0\right)d\wh{H}(t) -2G(1/2).
\end{eqnarray*}
Note that $\sig(u^Tx+b)\in\mathcal{G}^L(B)$  for any $u\in\mathbb{R}^p$ and any $b\in\mathbb{R}$ has already been proved in the proof of Theorem \ref{thm:mean-cov}. The same argument also leads to the same conclusion for $\ramp(u^Tx+b))$ for any $u\in\mathbb{R}^p$ and any $b\in\mathbb{R}$. Then, using the same argument in the proof of Theorem \ref{thm:cov-T3} ($\ramp(\cdot)$ is bounded between $0$ and $1$ just as $\sig(\cdot)$), we have
$$
\left|\int \ramp(z+\Delta/\delta_1)dH(z) - \int \ramp((\delta_2/\delta_1) z)d\wh{H}(z)\right|\lesssim \sqrt{\frac{p}{n}}+\epsilon,$$
with the choice $\kappa=O\left(\sqrt{\frac{p}{n}}+\epsilon\right)$, $B\asymp 1$, $L\asymp 1$ and $\delta=e^{-C'(p+n\epsilon^2)}$, where $\delta_1=\sqrt{v^T\Sigma v}$, $\delta_2=\sqrt{v^T\wh{\Sigma} v}$, $\Delta=v^T(\theta-\wh{\theta})$. Since $\int \ramp((\delta_2/\delta_1) z)d\wh{H}(z)=1/2=\int \ramp(z)dH(z)$ by $H(t)+H(-t)\equiv1$, the above inequality is equivalent to $|h(\Delta/\delta_1)-h(0)|\lesssim\sqrt{\frac{p}{n}}+\epsilon$, where $h(t)=\int\ramp(z+t)dH(z)$. Note that $h'(t)=\mathbb{P}_{Z\sim H}(|Z+t|\leq 1/2)$. By the condition $H\in\mathcal{H}(M')$, we have $\inf_{|t|\leq 1/6}|h'(t)|\geq 1/M'$. By the monotonicity of $h(t)$, as long as $|h(t)-h(0)|$ is sufficiently small, we have $|h(t)-h(0)|\geq (M')^{-1}|t|$, which implies
\begin{equation}
\frac{|v^T(\wh{\theta}-\theta)|}{\sqrt{v^T\Sigma v}}\lesssim \sqrt{\frac{p}{n}}+\epsilon. \label{eq:civic-type-r}
\end{equation}
Since $\delta_1\leq M^{1/2}$, we have $|v^T(\wh{\theta}-\theta)|\lesssim \sqrt{\frac{p}{n}}+\epsilon$. Taking supreme over all unit vector, we have $\|\wh{\theta}-\theta\|\lesssim\sqrt{p/n}+\epsilon$ with high probability.

To show the error bound for the scatter matrix estimator, we choose $w_1=\kappa/2$, $w_2=-\kappa/2$, $w_j=0$ for all $j\geq 3$, $g_1(x)=\ramp\left(\frac{v^T(x-\theta)}{\sqrt{v^T{\Sigma}v}}-\frac{1}{2}\right)$, and $g_2(x)=\ramp\left(-\frac{v^T(x-\theta)}{\sqrt{v^T{\Sigma}v}}-\frac{1}{2}\right)$ for some unit vector $v$ in (\ref{eq:focus-rs}). Since
$$\sum_{j\geq 1}w_jg_j(x)=\frac{\kappa}{2}R\left(\left|\frac{v^T(x-\theta)}{\sqrt{v^T{\Sigma}v}}\right|\right),$$
the inequality (\ref{eq:focus-rs}) becomes
$$\bar{f}\left(\kappa;\sqrt{v^T\Sigma v},\sqrt{v^T\wh{\Sigma}v},v^T(\theta-\wh{\theta})\right)\leq 2C\kappa(2B)^{L-1}\left(\sqrt{\frac{p}{n}}+\sqrt{\frac{\log(1/\delta)}{n}}\right) + 4C_1B^{L-1}\kappa\epsilon,$$
where
\begin{eqnarray*}
\bar{f}(t;\delta_1,\delta_2,\Delta) &=& \int S\left(\frac{1}{1+e^{-tR(|z|)/2}},1\right)dH(z) \\
&& + \int S\left(\frac{1}{1+e^{-tR(|(\delta_2/\delta_1)z-\Delta/\delta_1|)/2}},0\right)d\wh{H}(z) - 2G(1/2).
\end{eqnarray*}
Then, using the same argument in the proof of Theorem \ref{thm:cov-T3}, we have
$$
\left|\int R(|z|)dH(z) - \int R(|(\delta_2/\delta_1)z-\Delta/\delta_1|)d\wh{H}(z)\right|\lesssim \sqrt{\frac{p}{n}}+\epsilon,$$
with the choice $\kappa=O\left(\sqrt{\frac{p}{n}}+\epsilon\right)$, $B\asymp 1$, $L\asymp 1$ and $\delta=e^{-C'(p+n\epsilon^2)}$, where $\delta_1=\sqrt{v^T\Sigma v}$, $\delta_2=\sqrt{v^T\wh{\Sigma} v}$, $\Delta=v^T(\theta-\wh{\theta})$. Since $\left|R(|(\delta_2/\delta_1)z-\Delta/\delta_1|)-R(|(\delta_2/\delta_1)z|)\right|\lesssim |\Delta/\delta_1|$, we have
$$\left|\int R(|(\delta_2/\delta_1)z|)d\wh{H}(z) -\int R(|(\delta_2/\delta_1)z-\Delta/\delta_1|)d\wh{H}(z)\right|\lesssim |\Delta/\delta_1| \lesssim \sqrt{\frac{p}{n}}+\epsilon$$
by (\ref{eq:civic-type-r}). By triangle inequality, we get
$$
\left|\int R(|z|)dH(z) - \int R(|(\delta_2/\delta_1)z|)d\wh{H}(z)\right|\lesssim \sqrt{\frac{p}{n}}+\epsilon,$$
which can be written as $|\bar{h}(1)-\bar{h}(\delta_2/\delta_1)|\lesssim \sqrt{\frac{p}{n}}+\epsilon$, with $\bar{h}(t)=\int R(|tz|)d\wh{H}(z)$, because $\int R(|z|)dH(z)=\int R(|z|)d\Phi(z)=\int R(|z|)d\wh{H}(z)$ by the condition that $H,\wh{H}\in\mathcal{H}$. Since
\begin{eqnarray*}
\bar{h}'(t) &=& 2\int_0^{1/t}zd\wh{H}(z) \\
&\geq& 2\int_{1/(8t)}^{1/t}zd\wh{H}(z) \\
&\geq& \frac{1}{4t}\mathbb{P}_{Z\sim\wh{H}}\left(\frac{1}{8t}\leq Z\leq \frac{1}{t}\right),
\end{eqnarray*}
we have $\inf_{|t-1|\leq 1/2}|\bar{h}'(t)|\geq (8M')^{-1}$. Since $\bar{h}(t)$ is increasing for all $t>0$, as long as $|\bar{h}(t)-\bar{h}(1)|$ is sufficiently small, we have $|\bar{h}(t)-\bar{h}(1)|\geq (8M')^{-1}|t-1|$, which implies $|\delta_2/\delta_1-1|\lesssim \sqrt{\frac{p}{n}}+\epsilon$. Following the last several lines of the proof of Theorem \ref{thm:cov-T3}, we obtain the desired result.
\end{proof}

\subsection{Proofs of Auxiliary Lemmas}\label{sec:aux}

\begin{proof}[Proof of Lemma \ref{lem:complexity-T1}]
The bound for the class $\mathcal{T}_3$ was proved by Lemma 7.2 of \cite{gao2018robust}. The same bound also holds for the class $\mathcal{T}_1$ because $\mathcal{T}_1\subset\mathcal{T}_3$.
\end{proof}

\begin{proof}[Proof of Lemma \ref{lem:complexity-T2}]
Since $X_i\sim N(0,\Sigma)$, we can write $X_i=\Sigma^{1/2}Z_i$ with $Z_i\sim N(0,I_p)$. Define
$$f(Z_1,...,Z_n)=\sup_{T\in\mathcal{T}_2}\left|\frac{1}{n}\sum_{i=1}^n\log T(\Sigma^{1/2}Z_i)-\mathbb{E}\log T(\Sigma^{1/2}Z)\right|.$$
We show $f(Z_1,...,Z_n)$ is a Lipschitz function. We have
\begin{eqnarray}
\nonumber && |f(Z_1,...,Z_n)-f(Y_1,...,Y_n)| \\
\nonumber &\leq& \frac{1}{n}\sum_{i=1}^n\sup_{\|w\|_1\leq\kappa,\|u_j\|\leq 1}\Bigg|\log\sig\left(\sum_{j\geq 1}w_j\relu(u_j^T\Sigma^{1/2}Z_i)\right) \\
\nonumber && -\log\sig\left(\sum_{j\geq 1}w_j\relu(u_j^T\Sigma^{1/2}Y_i)\right)\Bigg| \\
\label{eq:Lip1} &\leq& \frac{1}{n}\sum_{i=1}^n\sup_{\|w\|_1\leq\kappa,\|u_j\|\leq 1}\left|\sum_{j\geq 1}w_j\relu(u_j^T\Sigma^{1/2}Z_i)-\sum_{j\geq 1}w_j\relu(u_j^T\Sigma^{1/2}Y_i)\right| \\
\nonumber &\leq& \frac{\kappa}{n}\sum_{i=1}^n\max_{j\geq 1}\sup_{\|u_j\|\leq 1}\left|\relu(u_j^T\Sigma^{1/2}Z_i)-\relu(u_j^T\Sigma^{1/2}Y_i)\right| \\
\nonumber &=& \frac{\kappa}{n}\sum_{i=1}^n\sup_{\|u\|\leq 1}\left|\relu(u^T\Sigma^{1/2}Z_i)-\relu(u^T\Sigma^{1/2}Y_i)\right| \\
\label{eq:Lip2} &\leq& \frac{\kappa}{n}\sum_{i=1}^n\sup_{\|u\|\leq 1}\left|u^T\left(\Sigma^{1/2}Z_i-\Sigma^{1/2}Y_i\right)\right| \\
\nonumber &\leq& \frac{M^{1/2}\kappa}{n}\sum_{i=1}^n\|Z_i-Y_i\| \\
\label{eq:Lip2.5} &\leq& \frac{M^{1/2}\kappa}{\sqrt{n}}\sqrt{\sum_{i=1}^n\|Z_i-Y_i\|^2}.
\end{eqnarray}
The inequalities (\ref{eq:Lip1}) and (\ref{eq:Lip2}) are implied by the fact that both the functions $\log\sig(\cdot)$ and $\relu(\cdot)$ have Lipschitz constants bounded by $1$.
Therefore, $f(Z_1,...,Z_n)$ is a Lipschitz function with Lipschitz constant $\frac{M^{1/2}\kappa}{\sqrt{n}}=O\left(\frac{\kappa}{\sqrt{n}}\right)$. By Talagrand's inequality \citep{talagrand1995concentration}, we have
$$f(Z_1,...,Z_n)\leq \mathbb{E}f(Z_1,...,Z_n)+ C\kappa\sqrt{\frac{\log(2/\delta)}{n}},$$
with probability at least $1-\delta$.

To bound $\mathbb{E}f(Z_1,...,Z_n)$, we use a standard symmetrization argument \citep{pollard2012convergence} and obtain the following bound that involves Rademacher complexity,
\begin{equation}
\mathbb{E}f(Z_1,...,Z_n) \leq 2\mathbb{E}\sup_{T\in\mathcal{T}_2}\left|\frac{1}{n}\sum_{i=1}^n\epsilon_i\log T(\Sigma^{1/2}Z_i)\right|,\label{eq:Rademacher-bound}
\end{equation}
where $\epsilon_1,...,\epsilon_n$ are i.i.d. uniform random variables on $\{-1,1\}$.
To bound the Rademacher complexity, we have
\begin{eqnarray}
\nonumber && \mathbb{E}\sup_{T\in\mathcal{T}_2}\left|\frac{1}{n}\sum_{i=1}^n\epsilon_i\log T(\Sigma^{1/2}Z_i)\right| \\
\nonumber &\leq& 2\mathbb{E}\sup_{\|w\|_1\leq\kappa,\|u_j\|\leq 1}\left(\frac{1}{n}\sum_{i=1}^n\epsilon_i\log\sig\left(\sum_{j\geq 1}w_j\relu(u_j^T\Sigma^{1/2}Z_i)\right)\right) \\
\label{eq:Lip3} &\leq& 2\mathbb{E}\sup_{\|w\|_1\leq\kappa,\|u_j\|\leq 1}\left(\frac{1}{n}\sum_{i=1}^n\epsilon_i\left(\sum_{j\geq 1}w_j\relu(u_j^T\Sigma^{1/2}Z_i)\right)\right) \\
\nonumber &\leq& 2\kappa\mathbb{E}\sup_{\|u\|\leq 1}\left|\frac{1}{n}\sum_{i=1}^n\epsilon_i\relu(u^T\Sigma^{1/2}Z_i)\right| \\
\nonumber &\leq& 4\kappa\mathbb{E}\sup_{\|u\|\leq 1}\left(\frac{1}{n}\sum_{i=1}^n\epsilon_i\relu(u^T\Sigma^{1/2}Z_i)\right) \\
\label{eq:Lip4} &\leq& 4\kappa\mathbb{E}\sup_{\|u\|\leq 1}\left(\frac{1}{n}\sum_{i=1}^n\epsilon_iu^T\Sigma^{1/2}Z_i\right) \\
\nonumber &\leq& 4\kappa M^{1/2}\mathbb{E}\left\|\frac{1}{n}\sum_{i=1}^n\epsilon_iZ_i\right\| \\
\label{eq:lucky} &\leq& 4\kappa M^{1/2}\sqrt{\frac{p}{n}}.
\end{eqnarray}
The inequalities (\ref{eq:Lip3}) and (\ref{eq:Lip4}) are by Theorem 7 of \cite{meir2003generalization}. The last inequality (\ref{eq:lucky}) is because $\mathbb{E}\left\|\frac{1}{n}\sum_{i=1}^n\epsilon_iZ_i\right\|\leq \sqrt{\mathbb{E}\left\|\frac{1}{n}\sum_{i=1}^n\epsilon_iZ_i\right\|^2}=\sqrt{\frac{p}{n}}$. The proof is complete by combining the bounds above.
\end{proof}

\begin{proof}[Proof of Lemma \ref{lem:complexity-T3}]
Let $f(X_1,...,X_n)=\sup_{T\in\mathcal{T}_3}\left|\frac{1}{n}\sum_{i=1}^n S(T(X_i),1)-\mathbb{E}S(T(X),1)\right|$. Since
\begin{eqnarray*}
&& \sup_{T\in\mathcal{T}_3}\sup_x\left|S(T(x),1)-G\left(\frac{1}{2}\right)-\frac{1}{2}G'\left(\frac{1}{2}\right)\right| \\
&\leq& \sup_{\left|t-\frac{1}{2}\right|\leq\kappa}\left|S(t,1)-G\left(\frac{1}{2}\right)-\frac{1}{2}G'\left(\frac{1}{2}\right)\right| \\
&\leq& \kappa\sup_{\left|t-\frac{1}{2}\right|\leq\kappa}\left|(1-t)G''(t)\right| \leq C_1\kappa,
\end{eqnarray*}
where we have used $\sup_{\left|t-\frac{1}{2}\right|\leq\kappa}\left|(1-t)G''(t)\right|\leq C_1$ because of the smoothness of $G(t)$ at $t=1/2$ by Condition \ref{cond:G}.
In the second last inequality above, we have used the fact that $\frac{\partial}{\partial t}S(t,1)=(1-t)G''(t)$ and $S(t,1)=G(t)+(1-t)G'(t)$. This implies that
$$\sup_{x_1,...,x_n,x_i'}\left|f(x_1,...,x_n)-f(x_1,...,x_{i-1},x_i',x_{i+1},...,x_n)\right|\leq \frac{2C_1\kappa}{n}.$$
Therefore, by McDiarmid's inequality \citep{mcdiarmid1989method}, we have
\begin{equation}
f(X_1,...,X_n)\leq\mathbb{E}f(X_1,...,X_n)+C_1\kappa\sqrt{\frac{2\log(1/\delta)}{n}},\label{eq:golf-r}
\end{equation}
with probability at least $1-\delta$. By the same argument of (\ref{eq:Rademacher-bound}), it is sufficient to bound the Rademacher complexity $\mathbb{E}\sup_{T\in\mathcal{T}_3}\left|\frac{1}{n}\sum_{i=1}^n\epsilon_iS(T(X_i),1)\right|$. Since $T\in\mathcal{T}_3$ implies $|T(X)-1/2|\leq\kappa$, the function $S(\sig(\cdot),1)$ has a Lipschitz constant bounded by $C_1$ on the domain of interest. By Theorem 7 of \cite{meir2003generalization}, we have
\begin{eqnarray}
\nonumber && \mathbb{E}\sup_{T\in\mathcal{T}_3}\left|\frac{1}{n}\sum_{i=1}^n\epsilon_iS(T(X_i),1)\right| \\
\label{eq:carrera-t} &\leq& 2C_1\mathbb{E}\sup_{\|w\|_1\leq\kappa,u_j\in\mathbb{R}^p,b_j\in\mathbb{R}}\left|\frac{1}{n}\sum_{i=1}^n\epsilon_i\sum_{j\geq 1}w_j\sig(u_j^TX_i+b_j)\right|.
\end{eqnarray}
By H\"{o}lder's inequality, we can further bound the above term by
$$2C_1\kappa\mathbb{E}\sup_{u\in\mathbb{R}^p,b\in\mathbb{R}}\left|\frac{1}{n}\sum_{i=1}^n\epsilon_i\sig(u^TX_i+b)\right|.$$
Define $\mathcal{D}=\left\{D(x)=\sig(u^Tx+b):u\in\mathbb{R}^p,b\in\mathbb{R}\right\}$, and then the Rademacher complexity can be bounded by Dudley's integral entropy, which gives
\begin{equation}
\mathbb{E}\sup_{D\in\mathcal{D}}\left|\frac{1}{n}\epsilon_iD(X_i)\right|\leq\mathbb{E}\frac{1}{\sqrt{n}}\int_0^2\sqrt{\log\mathcal{N}(\delta,\mathcal{D},\|\cdot\|_n)}d\delta,\label{eq:dudley}
\end{equation}
where $\mathcal{N}(\delta,\mathcal{D},\|\cdot\|_n)$ is the $\delta$-covering number of $\mathcal{D}$ with respect to the empirical $\ell_2$ distance $\|f-g\|_n=\sqrt{\frac{1}{n}\sum_{i=1}^n(f(X_i)-g(X_i))^2}$. Since the VC-dimension of $\mathcal{D}$ is $O(p)$, we have $\mathcal{N}(\delta,\mathcal{D},\|\cdot\|_n)\lesssim p(16e/\delta)^{O(p)}$ (see Theorem 2.6.7 of \cite{van1996weak}). This leads to the bound $\frac{1}{\sqrt{n}}\int_0^2\sqrt{\log\mathcal{N}(\delta,\mathcal{D},\|\cdot\|_n)}d\delta\lesssim \sqrt{\frac{p}{n}}$, which gives the desired result. 
\end{proof}

\begin{proof}[Proof of Lemma \ref{lem:complexity-T4}]
Following the proof of Lemma \ref{lem:complexity-T2}, we write $X_i=\Sigma^{1/2}Z_i$ with $Z_i\sim N(0,I_p)$. Define
$$f(Z_1,...,Z_n)=\sup_{T\in\mathcal{T}_4}\left|\frac{1}{n}\sum_{i=1}^nS(T(\Sigma^{1/2}Z_i),1)-\mathbb{E}S(T(\Sigma^{1/2}Z),1)\right|.$$
Since $T\in\mathcal{T}_4$ implies $|T(X)-1/2|\leq\kappa_1$, the function $S(\sig(\cdot),1)$ has a Lipschitz constant bounded by some constant $C_1$ on the domain of interest. Therefore,
\begin{eqnarray*}
&& \left|f(Z_1,...,Z_n)-f(Y_1,...,Y_n)\right| \\
&\leq& C_1\frac{1}{n}\sum_{i=1}^n\sup_{\|w\|_1\leq\kappa_1,\|v_j\|_1\leq\kappa_2,\|u_l\|\leq 1}\Bigg|\sum_{j\geq 1}w_j\sig\left(\sum_{l\geq 1}v_{jl}\relu(u_l^T\Sigma^{1/2}Z_i)\right) \\
&& - \sum_{j\geq 1}w_j\sig\left(\sum_{l\geq 1}v_{jl}\relu(u_l^T\Sigma^{1/2}Y_i)\right)\Bigg|.
\end{eqnarray*}
Then, by a successive argument of H\"{o}lder's inequalities and Lipschitz continuity similar to (\ref{eq:Lip1})-(\ref{eq:Lip2.5}), we have
$$\left|f(Z_1,...,Z_n)-f(Y_1,...,Y_n)\right|\leq C_2\frac{\kappa_1\kappa_2}{\sqrt{n}}\sqrt{\sum_{i=1}^n\|Z_i-Y_i\|^2}.$$
By Talagrand's inequality \citep{talagrand1995concentration}, we have
$$f(Z_1,...,Z_n)\leq \mathbb{E}f(Z_1,...,Z_n)+ C\kappa_1\kappa_2\sqrt{\frac{\log(2/\delta)}{n}},$$
with probability at least $1-\delta$.

To bound $\mathbb{E}f(Z_1,...,Z_n)$, it is sufficient to analyze the Rademacher complexity according to (\ref{eq:Rademacher-bound}). Again, this can be done by following the same argument that leads to (\ref{eq:lucky}), and thus we have $\mathbb{E}f(Z_1,...,Z_n)\lesssim \kappa_1\kappa_2\sqrt{\frac{p}{n}}$, which completes the proof.
\end{proof}

\begin{proof}[Proof of Lemma \ref{lem:complexity-deep}]
We first prove the result for $\mathcal{T}=\mathcal{T}^L(\kappa,B)$.
Let $f(X_1,...,X_n)=\sup_{T\in\mathcal{T}^L(\kappa,B)}\left|\frac{1}{n}\sum_{i=1}^n S(T(X_i),1)-\mathbb{E}S(T(X),1)\right|$.
For any $g\in\mathcal{G}_{\sig}$, we have $\sup_x|g(x)|\leq 1$. Suppose for any $g\in\mathcal{G}^l(B)$, $\sup_x|g(x)|\leq \tau$, then we have $\sup_x|g(x)|\leq B\tau$ for any $g\in\mathcal{G}^{l+1}(B)$ by H\"{o}lder's inequality. A mathematical induction argument then gives $\sup_x|g(x)|\leq B^{L-1}$ for any $g\in\mathcal{G}^L(B)$. Therefore, $T\in\mathcal{T}^L(\kappa,B)$ implies that $\sup_x|T(x)-1/2|\leq B^{L-1}\kappa$. By the same argument that derives (\ref{eq:golf-r}), we then have
$$f(X_1,...,X_n)\leq\mathbb{E}f(X_1,...,X_n)+C_1B^{L-1}\kappa\sqrt{\frac{2\log(1/\delta)}{n}},$$
with probability at least $1-\delta$.

It is sufficient to analyze the Rademacher complexity, and we have
\begin{eqnarray}
\nonumber && \mathbb{E}\sup_{T\in\mathcal{T}^L(\kappa,B)}\left|\frac{1}{n}\sum_{i=1}^n\epsilon_iS(T(X_i),1)\right| \\
\label{eq:carrera} &\leq& 2C_1\mathbb{E}\sup_{\|w\|_1\leq\kappa, g_j\in\mathcal{G}^L(B)}\left|\frac{1}{n}\sum_{i=1}^n\epsilon_i\sum_{j\geq 1}w_jg_j(X_i)\right| \\
\label{eq:carrera-s} &\leq& 2C_1\kappa\mathbb{E}\sup_{g\in\mathcal{G}^L(B)}\left|\frac{1}{n}\sum_{i=1}^n\epsilon_ig(X_i)\right| \\
\label{eq:gts} &\leq& 4C_1\kappa\mathbb{E}\sup_{\|v\|_1\leq B, g_h\in\mathcal{G}^{L-1}(B)}\left|\frac{1}{n}\sum_{i=1}^n\epsilon_i\sum_{h=1}^Hv_hg_h(X_i)\right| \\
\label{eq:gt3} &\leq& 4C_1B\kappa\mathbb{E}\sup_{g\in\mathcal{G}^{L-1}(B)}\left|\frac{1}{n}\sum_{i=1}^n\epsilon_ig(X_i)\right| \\
\label{eq:gt3rs} &\leq& 2C_1\kappa(2B)^{L-1}\mathbb{E}\sup_{g\in\mathcal{G}_{\sig}}\left|\frac{1}{n}\sum_{i=1}^n\epsilon_ig(X_i)\right|.
\end{eqnarray}
We explain each of the inequalities above. The first inequality (\ref{eq:carrera}) follows the same argument that derives (\ref{eq:carrera-t}). The inequalities (\ref{eq:carrera-s}) and (\ref{eq:gt3}) are implied by H\"{o}lder's inequality. We have used Theorem 7 of \cite{meir2003generalization} to derive (\ref{eq:gts}). Finally, (\ref{eq:gt3rs}) is from a mathematical induction argument. Note that $\mathbb{E}\sup_{g\in\mathcal{G}_{\sig}}\left|\frac{1}{n}\sum_{i=1}^n\epsilon_ig(X_i)\right|\lesssim \sqrt{\frac{p}{n}}$ can be derived from Dudley's integral entropy (\ref{eq:dudley}), and then we obtain the desired result.

The class $\bar{\mathcal{T}}^L(\kappa,B)$ only differs from ${\mathcal{T}}^L(\kappa,B)$ in the bottom layer. That is, we use $\sig(\cdot)$ in the bottom layer of ${\mathcal{T}}^L(\kappa,B)$ and $\ramp(\cdot)$ in the bottom layer of $\bar{\mathcal{T}}^L(\kappa,B)$. When we prove the result for ${\mathcal{T}}^L(\kappa,B)$, we use the following properties of $\sig(\cdot)$: it is a function bounded between $0$ and $1$; it is increasing; it has a bounded Lipschitz constant. All of the properties hold for $\ramp(\cdot)$, and thus the desired conclusion holds for $\bar{\mathcal{T}}^L(\kappa,B)$ as well.
\end{proof}

\section*{Acknowledgement}

The authors thank Zhaoran Wang for pointing out references on learning implicit models. The authors also thank Jiantao Jiao for pointing out references on scoring rules and for helpful discussion during this project.

\bibliographystyle{plainnat}
\bibliography{Robust}


\end{document}